\documentclass[lettersize,journal]{IEEEtran}
\usepackage{amsmath,amsfonts}
\usepackage{algorithmic}
\usepackage{algorithm}
\usepackage{array}
\usepackage[caption=false,font=normalsize,labelfont=sf,textfont=sf]{subfig}
\usepackage{textcomp}
\usepackage{stfloats}
\usepackage{url}
\usepackage{verbatim}
\usepackage{graphicx}
\usepackage{cite}
\usepackage{algorithm}
\usepackage{algorithmic}
\usepackage{amsmath}
\usepackage{amsthm}
\usepackage{amssymb}
\usepackage{nicefrac}
\newcommand{\algname}[1]{{\color{black}\small \sf #1}}

\usepackage{color}
\usepackage{amsmath}
\usepackage{amssymb}
\usepackage{amsthm}
\usepackage{caption}
\usepackage{subcaption}
\usepackage{mathtools}

\usepackage[flushleft]{threeparttable} %

\newtheorem{theorem}{Theorem}
\newtheorem{lemma}{Lemma}

\newtheorem{corollary}{Corollary}

\newtheorem{assumption}{Assumption}

\theoremstyle{definition}
\newtheorem{definition}{Definition}

\newtheorem{remark}{Remark}

\usepackage{pifont}
\newcommand{\cmark}{\ding{51}}
\newcommand{\xmark}{\ding{55}}

\newcommand{\sz}[1]{\left|#1\right|}

\newcommand{\norm}[1]{\left\lVert#1\right\rVert}

\definecolor{lightGreen}{RGB}{150,255,150}

\begin{document}

\title{Byzantine-Robust and Communication-Efficient Distributed Learning via Compressed Momentum Filtering}

\author{Changxin Liu, \IEEEmembership{Member, IEEE}, Yanghao Li, Yuhao Yi, and Karl H. Johansson, \IEEEmembership{Fellow, IEEE}
	%
	%
	\thanks{C. Liu and K. H. Johansson are with the School of Electrical Engineering and Computer Science, and Digital Futures, KTH Royal Institute of Technology, Sweden (Email: {\tt\small changxin, kallej@kth.se}).
		
		Y. Li and Y. Yi are with the College of Computer Science, Sichuan University, China (Email: {\tt\small li15583209175@gmail.com, yuhaoyi@scu.edu.cn}). Y. Yi is also with Institute of Clinical Pathology, West China Hospital, Sichuan University.
  
  }
	%
}


\markboth{Journal of \LaTeX\ Class Files,~Vol.~14, No.~8, August~2021}%
{Shell \MakeLowercase{\textit{et al.}}: A Sample Article Using IEEEtran.cls for IEEE Journals}


\maketitle

\begin{abstract}
Distributed learning has become the standard approach for training large-scale machine learning models across private data silos. While distributed learning enhances privacy preservation and training efficiency, it faces critical challenges related to Byzantine robustness and communication reduction. Existing Byzantine-robust and communication-efficient methods rely on full gradient information either at every iteration or at certain iterations with a probability, and they only converge to an unnecessarily large neighborhood around the solution. Motivated by these issues, we propose a novel Byzantine-robust and communication-efficient stochastic distributed learning method that imposes no requirements on batch size and converges to a smaller neighborhood around the optimal solution than all existing methods, aligning with the theoretical lower bound. 
Our key innovation is leveraging Polyak Momentum to mitigate the noise caused by both biased compressors and stochastic gradients, thus defending against Byzantine workers under information compression. 
We provide proof of tight complexity bounds for our algorithm in the context of non-convex smooth loss functions, demonstrating that these bounds match the lower bounds in Byzantine-free scenarios. Finally, we validate the practical significance of our algorithm through an extensive series of experiments, benchmarking its performance on both binary classification and image classification tasks.
\end{abstract}

\begin{IEEEkeywords}
Distributed machine learning, Byzantine robustness, communication compression, cyber-physical systems.
\end{IEEEkeywords}

\section{Introduction}


Data is crucial for machine learning, but traditional centralized data processing has become unfeasible due to privacy concerns and regulations, such as the General Data Protection Regulation (GDPR) in Europe. These challenges have necessitated a paradigm shift in how data is handled and processed for machine learning purposes. In response, the field has seen the development and widespread adoption of distributed learning algorithms
\cite{mcmahan2017communication,kairouz2021advances}. 
In this paradigm, participating machines collaborate with cloud-based computing platforms via communication networks to train high-performance machine learning models without disclosing local private data.

While distributed learning enhances privacy preservation and training efficiency, it faces two critical challenges. First, distributed learning systems, a special case of cyber-physical systems (CPS), are inherently vulnerable to misbehaving (a.k.a., {Byzantine}) workers \cite{liu2024survey}. Studies such as \cite{NEURIPS2019_ec1c5914,so2020byzantine,KHJ22,yang2024byzantine} have reported that a few Byzantine machines can severely degrade model performance by transmitting falsified information. Second, the cost of communicating dense gradient vectors poses a significant bottleneck. In many practical applications, communication takes much more time than computation, hindering the overall efficiency of the training system \cite{seide20141,alistarh2017qsgd}.

Motivated by these pressing needs,
the concepts of \textit{Byzantine robustness} and \textit{communication reduction} in privacy-preserving distributed learning have recently garnered increasing attention. 
In the literature, Byzantine robustness has primarily been pursued through the design of robust aggregation rules, such as coordinate-wise median (\algname{CWMed}) \cite{yin2018byzantine}, which allow the server to filter out information from potential Byzantine workers. Integrating these rules into various standard distributed learning algorithms has been extensively studied, resulting in Byzantine-robust learning algorithms with varying robustness guarantees \cite{yin2018byzantine,KHJ21,el2021distributed,FGGPS22}.
For reducing communication overhead, a leading strategy is compression, where dense vectors such as stochastic gradients, model parameters, or Hessians are compressed or sparsified before transmission \cite{seide20141,alistarh2018convergence,stich2018sparsified}.

Although Byzantine robustness and communication reduction have been widely researched, most studies address them in isolation, with few approaches integrating both aspects \cite{bernstein2018signsgd,ghosh2021communication,acharya2022robust,gorbunov2022variance,zhu2023byzantine_TSP,rammal2023communication}. 
Notably, \cite{ghosh2021communication,rammal2023communication} considered general biased contractive compressors but assumed full gradient information. In contrast, \cite{gorbunov2022variance} and \cite{zhu2023byzantine_TSP} partially relaxed the requirement for full gradient information but required the compression operator to be unbiased.
When training machine learning models on large datasets, it is crucial to account for the practical scenario involving both \textit{batch-free stochastic gradients} and \textit{biased contractive compressors} such as $\text{Top}_k$. However, integrating these elements presents significant challenges, as both biased compression and stochastic gradients introduce noise, complicating the defense against Byzantine attacks.




\subsection{Related Works}

\subsubsection{Byzantine-robust distributed learning} A distributed learning algorithm is considered Byzantine-robust if the model's performance remains accurate even in the presence of Byzantine workers, which may behave arbitrarily. The pipeline for Byzantine-robust distributed learning typically consists of three key steps \cite{guerraoui2024byzantine}: \textit{i)} pre-processing of vectors from workers (e.g., models or stochastic gradients), \textit{ii)} robust aggregation of these vectors, and \textit{iii)} application of an optimization method. Existing works in this field often differ in their approach to one or more of these steps. For pre-processing, strategies such as bucketing \cite{KHJ22} and Nearest Neighbor Mixing (\algname{NNM}) \cite{AFGGPS23} have been proposed. Robust aggregation rules include \algname{CWMed}, coordinate-wise trimmed mean (\algname{CWTM}) \cite{yin2018byzantine}, centered clipping \cite{KHJ21}, and minimum enclosing ball with outliers \cite{yi2024near}. 
A Byzantine attack identification strategy has also been shown to strengthen robustness in distributed computing for matrix multiplication tasks \cite{hong2022hierarchical}.
Various base optimization algorithms have been employed, including \algname{SGD} \cite{yang2021basgd}, Polyak Momentum \cite{el2021distributed,KHJ22,FGGPS22}, \algname{SAGA} \cite{wu2020federated}, and \algname{VR-MARINA} \cite{gorbunov2022variance}. 

The best achievable accuracy of Byzantine-robust distributed learning algorithms is inherently limited by the characteristics of the training datasets provided by honest workers. In \textit{heterogeneous} setups \cite{LXCGL19,DD21}, where data distributions across workers may not accurately represent the overall population, converging to the exact optimal model becomes theoretically impossible. This limitation arises because Byzantine workers can exploit the inherent data heterogeneity to inject falsified information while remaining undetected. For a detailed lower bound result, see~\cite{KHJ21,AFGGPS23}.


\subsubsection{Byzantine-robust learning under information compression} 
In \cite{bernstein2018signsgd}, majority voting was integrated with \algname{signSGD} (with and without Polyak Momentum) to achieve both Byzantine robustness and communication compression. However, it is known that \algname{signSGD} does not always converge \cite{karimireddy2019error}. Leveraging general unbiased compressors, \cite{zhu2023byzantine_TSP} proposed a communication-efficient and Byzantine-robust algorithm based on the variance-reduced method \algname{SAGA} and the gradient quantization framework \algname{DIANA} \cite{horvath2023stochastic}. This approach was further developed in \cite{gorbunov2022variance} with another variance-reduced method \algname{VR-MARINA}, establishing improved convergence guarantees under more relaxed assumptions. More practical contractive compressors, such as $\text{Top}_k$ quantization, which are typically biased and used in combination with the advanced technique of error feedback \cite{seide20141,richtarik2021ef21}, have also been explored recently to address this challenge. However, all existing works in this area have considered the limited case of full gradient \cite{ghosh2021communication,rammal2023communication}. 

Another closely related work is \cite{acharya2022robust}, where $\text{Top}_k$ compression was applied at the server, rather than at the workers, to reduce the computational load of the geometric median operation. Additionally, a memory augmentation mechanism similar to error feedback was incorporated to enhance performance.

\subsection{Main Contributions}
We summarize our main contributions below.
\subsubsection{The first Byzantine-robust stochastic distributed learning method with error feedback} We propose the first Byzantine-robust distributed stochastic learning method utilizing practical biased compressors. Similar to \cite{ghosh2021communication,rammal2023communication}, we employ biased compression in conjunction with the error feedback strategy \algname{EF21} from \cite{richtarik2021ef21}. However, unlike \cite{ghosh2021communication,rammal2023communication}, which rely on full gradient information, our new algorithm is batch-free and leverages Polyak Momentum to mitigate the noise from both stochastic gradients and biased compression, ultimately enhancing the defense against Byzantine workers.

\subsubsection{New complexity results} We establish the complexity bounds of our new algorithm under standard assumptions. These complexity results demonstrate that our algorithm outperforms the state-of-the-art in the specific case of full gradients \cite{rammal2023communication}. Indeed, our complexity bounds are tight, as they match the lower bound results in both the stochastic and full gradient scenarios when the problem is Byzantine-free.

\subsubsection{Smaller size of the neighborhood} Under the standard $G^2$-heterogeneity assumption, our new algorithm converges to a smaller neighborhood around the optimal solution than all existing competitors, aligning with the lower bound \cite{KHJ21,AFGGPS23}. For detailed comparisons, see Table \ref{table1}.

The remainder of the paper is organized as follows: In Section 2, we formulate the problem and introduce key technical preliminaries. In Section 3, we present our new algorithm along with its convergence rate guarantees. In Section 4, we provide empirical results, followed by the conclusion in Section 5.
\begin{table*}[t]
	\centering
	\caption{Summary of related works on Byzantine-robust and communication-efficient distributed methods. ``Assumption": additional assumptions beyond the smoothness of $\mathcal{L}_i, i \in \mathcal{H}$ and bounded heterogeneity among honest workers. ``Complexity": the total number of communication rounds required for each worker to find $x$ such that $\mathbb{E}\left[\lVert \nabla\mathcal{L}_{\mathcal{H}} ({x}) \rVert \right] \leq \varepsilon$. 
		S.C. stands for $\mu$-strong convexity.
		$\sigma^2$ denotes the variance of local stochastic gradients, $\kappa$ represents the parameter of robust aggregators, $\alpha \in (0,1]$ and $\omega \geq 0$ stand for the parameter for biased contractive and unbiased compressors, respectively, $G^2$ is the bound of heterogeneity among honest workers. $p\in(0,1]$ is the sample probability used in \algname{Byz-VR-MARINA}. $m$ is the size of local dataset on workers for \algname{Byz-VR-MARINA} and \algname{BROADCAST}. 
	}
	\label{table1}
	\begin{threeparttable}
		\centering
		\begin{tabular}{cccccc}
			\hline
			Method &
			Assumption & Compressor &
			Batch-free? &
			Complexity & Accuracy
			\\
			\hline
			\begin{tabular}{c}
				\algname{BROADCAST} \tnote{\color{blue}(1)}\\
				\cite{zhu2023byzantine_TSP}   
			\end{tabular} &
			\begin{tabular}{c}
				{ $\mathcal{L}_i$ is finite-} \\
				{ sum and S.C.}
			\end{tabular}
			& unbiased &
			\xmark& $\frac{m^2{(1+\omega)^{\nicefrac{3}{2}}}(n-f)}{\mu^2(n-2f)}$
			& { ${\kappa (1+\omega) G^2}$} \\
			\hline
			\begin{tabular}{c}
				\algname{Byz-VR-MARINA} \tnote{\color{blue}(2)}\\
				\cite{gorbunov2022variance} 
			\end{tabular} &
			\begin{tabular}{c}
				$\mathcal{L}_i$ is \\
				finite-sum 
			\end{tabular}
			& unbiased &
			\xmark& 
			{ $\frac{\left( 1+ \sqrt{\max\{ \omega^2, {m\omega } \}} \left( \sqrt{\frac{1}{n-f}}+\sqrt{\kappa \max\{ \omega, {m} \} }  \right)\right)}{\varepsilon^2}$} & $\frac{\kappa G^2}{p}$ \\
			\hline
			\begin{tabular}{c}
				\algname{Byz-EF21} \tnote{\color{blue}(2)}\\
				\cite{rammal2023communication} 
			\end{tabular}
			&
			\begin{tabular}{c}
				full\\
				gradient
			\end{tabular} & \begin{tabular}{c}
				biased \\
				contractive
			\end{tabular} &
			\xmark &
			$\frac{1+\sqrt{\kappa}}{\alpha \varepsilon^2}$ & {$(\kappa+\sqrt{\kappa})G^2$} \\
			\hline
			\begin{tabular}{c}
				\algname{Byz-EF21-SGDM}\\
				(This work)
			\end{tabular} & 
			\begin{tabular}{c}
				bounded \\
				variance
			\end{tabular}
			& \begin{tabular}{c}
				biased \\
				contractive
			\end{tabular}
			&
			\cmark &
			\begin{tabular}{c}
				$\frac{ \sigma^2 }{(n-f)\varepsilon^4}+ \frac{\kappa \sigma^2 }{\varepsilon^4} $ \\
				{$\frac{\sqrt{\kappa+1}}{\alpha \varepsilon^2} $} (full gradient)
			\end{tabular} & $\kappa G^2$
			\\
			\hline
		\end{tabular}
		\begin{tablenotes}
			\item[{\color{blue} (1)}] \algname{BROADCAST} relies on a specific aggregation method, namely the geometric median. In this context, we derive the lower bound for the achievable accuracy by leveraging the property that geometric median is $(f,\kappa)$-robust with $\kappa = (1+\frac{f}{n-2f})^2$.
			\item[{\color{blue} (2)}] For comparison, the complexity results of \algname{Byz-VR-MARINA} and \algname{Byz-EF21} are derived by exploring the relationship between $(\delta,c)$-agnostic robust aggregator and $(f,\kappa)$-robust aggregator. 
			See Remark \ref{remark:full_gradient} for details.
		\end{tablenotes}
\end{threeparttable}
\end{table*}

\section{Problem Statement and Preliminaries}


Consider a server-worker distributed learning system with one central server and $n$ workers. Each worker $i\in [n]$ possesses a local dataset $\mathcal{D}_i$. The distributions $\mathcal{D}_1, \mathcal{D}_2, \dots, \mathcal{D}_n$ may differ arbitrarily.
These workers collaborate with the central server to train a model, parameterized by $x \in\mathbb{R}^d$, by solving
\begin{equation}\label{eq:original_P}
\begin{split}
	\min_{x\in\mathbb{R}^d} \, \left\{ \mathcal{L}(x)= \frac{1}{n} \sum_{i=1}^n  \mathcal{L}_i(x) \right\}
\end{split}
\end{equation}
where 
\begin{equation*}
\mathcal{L}_i(x) = \mathbb{E}_{\xi_{i}\sim \mathcal{D}_i}  \ell(x, \xi_{i})
\end{equation*}
and $\ell$ represents the loss over a single data point.
We consider an adversarial setting where the server is honest, but  $f$  out of  $n$  workers exhibit Byzantine behavior, with their identities being unknown \cite{lamport2019byzantine}. Throughout this work, we denote by $\mathcal{H}$ the set of indices of honest workers.

Due to the presence of Byzantine workers, solving \eqref{eq:original_P} is neither reasonable nor generally possible. Instead, a more reasonable goal is to approximate a stationary point of the following cost function:
\begin{equation}
\label{eq:shifted_problem}
\mathcal{L}_{\mathcal{H}}(x)=\frac{1}{n-f} \sum_{i\in\mathcal{H}}\mathcal{L}_i(x).
\end{equation}
We define an algorithm as Byzantine-robust if it can find an $\varepsilon$-approximate stationary point for $\mathcal{L}_{\mathcal{H}}$ despite the presence of $f$ Byzantine workers. In particular, we introduce the concept of Byzantine robustness as follows.

\begin{definition}[\bf{$(f,\varepsilon)$-Byzantine robustness}]
\label{def:Byzantine_robustness}
A learning algorithm is said $(f,\varepsilon)$-Byzantine robust if, even in the presence of $f$ Byzantine workers, it outputs $\hat{x}$ satisfying
\begin{equation*}
	\mathbb{E}\left[\lVert \nabla\mathcal{L}_{\mathcal{H}} (\hat{x}) \rVert^2 \right] \leq \varepsilon
\end{equation*}
where $\mathbb{E}[\cdot]$ is defined over the randomness of the algorithm.
\end{definition}

Note that $(f,\varepsilon)$-Byzantine robustness for any $\varepsilon$ is generally not possible when $f\geq n/2$ \cite{liu2021approximate}. Therefore, we assume an upper bound for the number of Byzantine workers $f<n/2$ in this work. 

In this work, we aim to design a Byzantine-robust distributed stochastic learning method for solving \(\eqref{eq:shifted_problem}\), with compressed communication between the server and workers.

\paragraph{Standard Byzantine-robust methods}
Byzantine-robust distributed learning algorithms comprise two primary components: robust aggregation and base optimization method. 

For robust aggregation, several effective methods include  \algname{CWTM} \cite{yin2018byzantine} and centered clipping \cite{KHJ21}. 
To quantify aggregation robustness, we introduce the concept of $(f,\kappa)$-robustness \cite{AFGGPS23}, which refers to the property that, for any subset of inputs of size $n-   f$, the output of the aggregation rule is in close proximity to the average of these inputs. This notion acts as a metric for evaluating the robustness of various aggregation rules; see \cite{AFGGPS23} for a detailed quantification of common aggregators.


\begin{definition}[\bf{$(f,\kappa)$-robustness}]
Given an integer $f< n/2$ and a real number $\kappa \geq 0$, an aggregation rule $F$ is $(f,\kappa)$-robust if for any set of $n$ vectors $\{g_1,g_2,\dots, g_n\}$, 
and any subset $S\subseteq [n]$ with $\sz{S} = n-f$,
\begin{align}
	\norm{F(g_1,g_2,\dots, g_n) - \overline{g}_S}^2 \leq \frac{\kappa}{\sz{S}}\sum_{i\in S}\norm{g_i - \overline{g}_S}^2\,
\end{align}
where $\overline{g}_S := {(n-f)^{-1}}\sum_{i\in S}g_i$.
\end{definition}
We note that pre-aggregation techniques such as bucketing \cite{KHJ22} and \algname{NNM} \cite{AFGGPS23} have proven effective in improving Byzantine robustness both theoretically and practically.

It has also been revealed that reducing the variance of honest gradients is beneficial to defend against Byzantine agents \cite{wu2020federated}. Such idea has been explored by designing Byzantine-robust methods based on \algname{SAGA} \cite{wu2020federated}, Polyak Momentum \cite{el2021distributed,FGGPS22}, and \algname{VR-MARINA} \cite{gorbunov2022variance}.


\paragraph{Communication compression} Communication compression techniques, such as quantization \cite{alistarh2018convergence,stich2018sparsified}, are highly effective in reducing the communication overhead of distributed learning methods. Among these techniques, (biased) contractive compressors stand out as the most versatile and practically useful class of compression mappings.
\begin{definition}[Contractive compressors]\label{def:contractive}
A (possibly randomized) mapping $\mathcal{C}:\mathbb{R}^d\rightarrow \mathbb{R}^d$ is called a contractive compression operator if there exists a constant $\alpha \in (0,1]$ such that 
\begin{equation*}
	\mathbb{E} [ \lVert \mathcal{C}(z)-z \rVert^2] \leq (1-\alpha) \lVert z \rVert^2, \,\, \forall z\in\mathbb{R}^d.
\end{equation*}
\end{definition}

The contractive condition in Definition \ref{def:contractive} is naturally satisfied by various compressors. Notable examples include: \textit{i)} the $\text{Top}_k$ sparsifier \cite{stich2018sparsified}, which retains the $k$ largest components of $z$ in magnitude and sets the remaining entries to zero, and \textit{ii)} the $\text{Rand}_k$ sparsifier \cite{beznosikov2023biased}, which preserves a randomly chosen subset of $k$ entries of $z$ and sets the remaining coordinates to zero, and then scales the sparisified vector by $d/k$. 
For an overview of both biased and unbiased compressors, refer to the summary provided in \cite{beznosikov2023biased}.

\paragraph{Brittleness of existing communication-efficient and Byzantine-robust solutions}
Addressing the critical needs of both Byzantine robustness and communication compression in heterogeneous setups is challenging, as communication compression introduces noise that further complicates defense against Byzantine workers. 
Indeed, existing approaches either focused on unbiased compressors \cite{gorbunov2022variance,zhu2023byzantine_TSP} and/or assumed full gradient information \cite{ghosh2021communication,rammal2023communication}. Consequently, none of them are applicable to the practically useful setting with batch-free stochastic gradients and biased contractive compression, which is widely employed in modern machine learning with large datasets.

Moreover, the state-of-the-art methods suffer from relatively large optimization errors and do not align with the lower bound provided in \cite{KHJ21,AFGGPS23} for Byzantine-robust distributed learning with heterogeneous datasets. It remains unclear whether this lower bound is still achievable under information compression.





\section{Communication-Efficient and Byzantine-Robust Distributed Learning}\label{sec:main}

In this section, we introduce our new algorithm and provide its rate of convergence.

\subsection{Algorithm description}

Communication compression introduces an additional layer for Byzantine workers to exploit, beyond the inherent heterogeneity in the system. Recent communication-efficient and Byzantine-robust methods \cite{zhu2023byzantine_TSP,gorbunov2022variance} addressed this challenge by developing algorithms based on the variance-reduced methods \algname{SAGA} and \algname{VR-MARINA}. However, these approaches still require full gradients at some iterations and only accommodate unbiased compressors due to the sensitivity of their base methods to biased gradient estimates. To overcome this limitation, we explore the use of Polyak Momentum, which imposes no requirements on the batch size and also exhibits a variance reduction effect \cite{el2021distributed,FGGPS22}.

We summarize our new method \algname{Byz-EF21-SGDM} in Algorithm \ref{alg:Byz-EF21-SGDM}. This algorithm builds on the recently proposed Error Feedback Enhanced with Polyak Momentum (\algname{EF21-SGDM}) \cite{fatkhullin2024momentum}. At each iteration $t$, the server applies an $(f,\kappa)$-robust aggregation over $g_i^{(t)}$ to compute  $g^{(t)}$, and updates the model parameter as  $x^{(t+1)} = x^{(t)} - \gamma g^{(t)}$, where $\gamma$ is the step-size. Then, the server broadcasts the updated model to all workers. Upon receiving  $x^{(t+1)}$, honest workers estimate their local momentums as  $v_i^{(t+1)} = (1 - \eta) v_i^{(t)} + \eta \nabla_x \ell_i(x^{(t+1)}, \xi_i^{(t+1)})$ (line 5), compress the changes in momentum  $c_i^{(t+1)} = \mathcal{C}(v_i^{(t+1)} - g_i^{(t)})$, and send these compressed vectors to the server (line 6). Concurrently, honest workers update their local state  $g_i^{(t+1)} = g_i^{(t)} + c_i^{(t+1)}$ based on the compressed change (line 7). The server also maintains a copy of $g_i^{(t)}$ for each worker and updates it upon receiving the compressed vectors (line 9).

Three crucial aspects of the proposed algorithm are introduced as follows. \textit{First}, our new algorithm utilizes stochastic gradients without imposing any requirements on batch size. This marks an improvement over existing competitors, which either require full gradients at every iteration \cite{ghosh2021communication,rammal2023communication} or at some iterations with a certain probability \cite{zhu2023byzantine_TSP,gorbunov2022variance}.
\textit{Second}, honest workers transmit only the compressed changes in their local true and estimated momentum variables to the server, specifically $\mathcal{C}(v_i^{(t+1)}-g_i^{(t)})$. This approach reduces communication overhead and helps exclude Byzantine workers who deviate from the algorithm by sending dense vectors. \textit{Third}, robust aggregation is performed on local momentum variables, leveraging the variance reduction effect to enhance the defense against Byzantine workers.

\begin{algorithm}[t]
\caption{\algname{Byz-EF21-SGDM}}
\label{alg:Byz-EF21-SGDM}
\textbf{Input}: initial model $x^{(0)}$, step-size $\gamma>0$, momentum coefficient $\eta\in(0,1]$, robust aggregation rule $F$, 
the number of rounds $T$\\
\textbf{Output}: $\hat{x}^{(T)}$ sampled uniformly from $x^{(0)},x^{(1)},\dots,x^{(T-1)}$\\
\textbf{Initialization}: 
for every honest worker $i\in \mathcal{H}$, 
$v_i^{(0)} = g_i^{(0)} = \nabla_x \ell_i(x^{(0)},\xi_{i}^{(0)})$, 
each worker $i\in[n]$ sends $g_i^{(0)}$ to the server

\begin{algorithmic}[1] 
	\FOR{$t= 0,1,\ldots, T-1$}
	\STATE Server computes
	$g^{(t)} = F(\{g_1^{(t)}, \ldots, g_n^{(t)}\})$ and $x^{(t+1)} = x^{(t)} - \gamma g^{(t)}$ 
	\STATE Server broadcasts $x^{(t+1)}$ to all workers
	\FOR{every honest worker $i\in \mathcal{H}$ in parallel}
	\STATE Estimate local momentum $    v_i^{(t+1)}=(1-\eta) v_i^{(t)} +  \eta \nabla_x \ell_i(x^{(t+1)},\xi_i^{(t+1)})$
	\STATE Compress $c_i^{(t+1)} = \mathcal{C}(v_i^{(t+1)}- g_i^{(t)})$ and send $c_i^{(t+1)}$ to the server
	\STATE Update local state 
	$  g_i^{(t+1)} = g_i^{(t)}  + c_i^{(t+1)} $
	\ENDFOR
	\STATE Server updates $g_i^{(t)},i\in[n]$ according to $g_i^{(t+1)} = g_i^{(t)}  + c_i^{(t+1)} $ 
	\ENDFOR
\end{algorithmic}
\end{algorithm}

\subsection{Rate of convergence}

To prove the convergence of \algname{Byz-EF21-SGDM}, we investigate the interaction between Polyak Momentum and robust aggregation under information compression. Specifically, we demonstrate that the error in stochastic gradients due to Byzantine workers depends on the compression error, the momentum deviation from the gradient, and the heterogeneity (see Lemma \ref{lem:robust_aggregation} in the Appendix). Furthermore, we observe that both the compression error and the momentum deviation can be bounded by $\lVert x^{(t+1)} - x^{(t)} \rVert^2$ multiplied by a constant, along with an additional term arising from heterogeneity (see Lemmas \ref{lem:compression_error} and \ref{lem:momentum} in the Appendix). These error terms can be managed using the well-known descent lemma \cite{li2021page}, which includes a term proportional to $-\lVert x^{(t+1)} - x^{(t)} \rVert^2$. As another contribution, we reveal that the lower bound for the best achievable accuracy, established for standard Byzantine-robust methods \cite{KHJ21,AFGGPS23}, is also attainable by our algorithm under information compression.

Before proceeding to the main result, we introduce three standard assumptions. We start from presenting the standard smoothness assumption.

\begin{assumption}[{Smoothness}]\label{assump:smoothness}
We assume $\mathcal{L}_{\mathcal{H}}$ is $L$-smooth, that is, $\lVert \nabla \mathcal{L}_{\mathcal{H}}(x)- \nabla \mathcal{L}_{\mathcal{H}}(y) \rVert  \leq  L \lVert x-y \rVert, \,\, \forall x,y \in \mathbb{R}^d$, and each $\mathcal{L}_i$ is $L_i$-smooth. We denote $\tilde{L}=(n-f)^{-1}\sum_{i\in\mathcal{H}}L_i^2$. In addition, we assume that $\mathcal{L}_{\mathcal{H}}$ is lower bounded, i.e., $\mathcal{L}_{\mathcal{H}}^* := \min_{x \in \mathbb{R}^d} \mathcal{L}_{\mathcal{H}}(x)>-\infty$.
\end{assumption}

To ensure provable Byzantine robustness, it is necessary to assume bounded heterogeneity among honest workers. Without this assumption, Byzantine workers could transmit arbitrary vectors while remaining undetected by pretending to have non-representative data.

\begin{assumption}[{Bounded heterogeneity}]\label{assump:bounded_hetero}
There exists a non-negative $G$ such that
\begin{equation*}
	\frac{1}{ n-f}\sum_{i\in\mathcal{H}} \lVert \nabla \mathcal{L}_i(x)- \nabla \mathcal{L}_{\mathcal{H}}(x) \rVert^2 \leq  G^2, \,\, \forall x\in\mathbb{R}^d.
\end{equation*}
	\end{assumption}

	\begin{assumption}[{Bounded variance}]\label{assump:bounded_var}
For each honest worker $i \in \mathcal{H}$, there exists $\sigma >0$ such that
\begin{equation*}
	\mathbb{E} \left[ \lVert \nabla_{x} \ell_i(x,\xi_i)-\nabla \mathcal{L}_i(x) \rVert^2  \right]\leq \sigma^2, \,\, \forall x \in\mathbb{R}^d
\end{equation*}
where $ \xi_i \sim \mathcal{D}_i$ are i.i.d. random samples.
\end{assumption}


Our convergence rate analysis is dependent on the following Lyapunuov function:
\begin{equation}\label{eq:lyapunuov}
\begin{split}
	\Gamma^{(t)} = &\delta^{(t)}+ 
	\frac{6\gamma(4\eta^2(1+\eta)(1+6\kappa)+3\kappa \alpha^2)}{\eta \alpha^2(n-f)}
	\sum_{i\in\mathcal{H}} \left\lVert M_i^{(t)} \right\rVert^2\\
	&+\frac{3\gamma}{\eta}\left\lVert \widetilde{M}^{(t)} \right\rVert^2 
\end{split}
\end{equation}
where $\delta_t =  \mathcal{L}_{\mathcal{H}}(x^{(t)})-\mathcal{L}_{\mathcal{H}}^* $, $M_i^{(t)}=  {v}_i^{(t)} - \nabla \mathcal{L}_i(x^{(t)})$, and $\widetilde{M}^{(t)} = (n-f)^{-1}\sum_{i\in\mathcal{H}} v_i^{(t)} - \nabla \mathcal{L}_{\mathcal{H}}(x^{(t)})$. The main convergence result for general non-convex functions is presented in Theorem \ref{thm:rate}, with its proof provided in the Appendix \ref{appdxSec:proofThm}.

\begin{theorem}\label{thm:rate}
Suppose Assumptions \ref{assump:smoothness}, \ref{assump:bounded_hetero}, and \ref{assump:bounded_var} hold.
For Algorithm \ref{alg:Byz-EF21-SGDM} applied to solve the distributed learning problem \eqref{eq:original_P} in the presence of $f<n/2$ Byzantine workers and communication compression with parameter $\alpha \in (0,1]$ defined in Definition \ref{def:contractive}, if $\eta\leq 1$ and
\begin{equation*}
	\gamma \leq \min \left\{ \frac{\alpha}{8\tilde{L}\sqrt{3(6\kappa+1)}},   \frac{\eta}{2\sqrt{3(6\kappa\tilde{L}^2+L^2)}}\right\},
\end{equation*}
then
\begin{equation}\label{eq:bound}
	\begin{split}
		& \mathbb{E} \left[ \lVert \nabla \mathcal{L}_{\mathcal{H}}(\hat{x} ^{(T)}) \rVert^2 \right]   \leq  \frac{\Gamma_0 }{\gamma T}  + \Delta \sigma^2  + 18\kappa G^2
	\end{split}
\end{equation}
where $\hat{x}^{(T)}$ is sampled uniformly at random from $x^{(0)},x^{(1)},\dots,x^{(T-1)}$, $\Gamma^{(0)}$ is defined in \eqref{eq:lyapunuov}, and
$
\Delta = {24\eta^3 (6\kappa +1) }/{\alpha^2}  + {6(6\kappa+1)\eta^2}/{\alpha}  + {3\eta}/{(n-f)}+ 18\kappa\eta.
$
By setting momentum 
$
\eta \leq \min \biggl\{ \left( \frac{L\delta_0\alpha^2}{24(1+6\kappa )\sigma^2 T} \right)^{\nicefrac{1}{4}}, \left( \frac{L\delta_0\alpha}{6(1+6\kappa)\sigma^2 T} \right)^{\nicefrac{1}{3}},  
$
$\left( \frac{L\delta_0(n-f)}{3\sigma^2 T} \right)^{\nicefrac{1}{2}} ,\left( \frac{L\delta_0}{18\kappa\sigma^2 T} \right)^{\nicefrac{1}{2}} 
\biggl\}
$,
we obtain
\begin{equation*}
	\begin{split}
		& \mathbb{E} \left[ \lVert \nabla \mathcal{L}_{\mathcal{H}}(\hat{x} ^{(T)}) \rVert^2 \right]  \\
		&\leq  \frac{\Gamma_0 }{\gamma T}  + \left( \frac{(24(1+6\kappa))^{\nicefrac{1}{3}}L\delta_0\sigma^{\nicefrac{2}{3}}}{\alpha^{\nicefrac{2}{3}} T} \right)^{\nicefrac{3}{4}} +  \left( \frac{18\kappa L\delta_0  \sigma^2}{T} \right)^{\nicefrac{1}{2}} \\
		& \quad +  \left( \frac{3L\delta_0 \sigma^2}{(n-f)T} \right)^{\nicefrac{1}{2}}  + \left( \frac{ \sqrt{6(1+6\kappa)} L\delta_0 \sigma}{\sqrt{\alpha}T} \right)^{\nicefrac{2}{3}}  + 18\kappa G^2.
	\end{split}
\end{equation*}
\end{theorem}

Theorem \ref{thm:rate} reveals the convergence of $\mathbb{E} \left[ \lVert \nabla \mathcal{L}_{\mathcal{H}}(\hat{x} ^{(T)}) \rVert^2 \right]$ by the proposed algorithm \algname{Byz-EF21-SGDM}, which aligns with the Byzantine robustness metric in Definition \ref{def:Byzantine_robustness}. We highlight several important properties of Theorem \ref{thm:rate}. This theorem provides the first theoretical result demonstrating the convergence of the error feedback method with stochastic gradients under Byzantine attacks. In the heterogeneous case (i.e.,  $G > 0$), the algorithm does not guarantee that $\mathbb{E} \left[ \lVert \nabla \mathcal{L}_{\mathcal{H}}(\hat{x} ^{(T)}) \rVert\right]$ can be made arbitrarily small. This limitation is inherent to all Byzantine-robust algorithms in heterogeneous settings. Specifically, with an order-optimal robustness coefficient \(\kappa = \mathcal{O}(\nicefrac{f}{n})\), such as with \algname{CWTM}, the result aligns with the lower bound \(\Omega(\nicefrac{f}{n}G^2)\) established by \cite{AFGGPS23}. The best achievable accuracy by \algname{Byz-EF21-SGDM} is tighter than those by \algname{Byz-VR-MARINA} and \algname{Byz-EF21} (see Table \ref{table1}).

As a consequence of Theorem \ref{thm:rate}, we provide complexity results for \algname{Byz-EF21-SGDM} in Corollary \ref{coro:stochastic_gradient}.

\begin{corollary}\label{coro:stochastic_gradient}
$\mathbb{E} \left[ \lVert \nabla \mathcal{L}_{\mathcal{H}}(\hat{x} ^{(T)}) \rVert \right]\leq \varepsilon$ after 
\begin{equation*}
	\begin{split}
		T= &\mathcal{O}\bigg(\frac{\tilde{L}\sqrt{\kappa+1}}{\alpha \varepsilon^2} + 
		\frac{(\kappa+1)^{\nicefrac{1}{3}}L\sigma^{\nicefrac{2}{3}}}{\alpha^{\nicefrac{2}{3}} \varepsilon^{\nicefrac{8}{3}}}   +    \frac{\sqrt{\kappa+1}L \sigma}{\alpha^{\nicefrac{1}{2}}\varepsilon^3}  \\
		& \quad \quad +   \frac{((n-f)\kappa+1) L  \sigma^2 }{(n-f)\varepsilon^4}
		\bigg)
	\end{split}
\end{equation*}
iterations.
\end{corollary}

Corollary \ref{coro:stochastic_gradient} provides the first sample complexity result for a Byzantine-robust stochastic method with error feedback. In the absence of Byzantine faults (i.e., $\kappa=0$), this corollary yields an asymptotic sample complexity of $\mathcal{O}(\nicefrac{L\sigma^2}{(n-f)\varepsilon^4})$ in the regime $\varepsilon \rightarrow 0$, which is optimal \cite{huang2022lower,fatkhullin2024momentum}.

Next, we consider a special case where local full gradients are available to workers, i.e., $\sigma=0$.

\begin{corollary}\label{coro:full_gradient}
If $\sigma=0$, then $\mathbb{E} \left[ \lVert \nabla \mathcal{L}_{\mathcal{H}}(\hat{x} ^{(T)}) \rVert \right]\leq \varepsilon$ after $T=\mathcal{O}(\nicefrac{\tilde{L}\sqrt{\kappa+1}}{\alpha \varepsilon^2})$ iterations.
\end{corollary}

\begin{remark}\label{remark:full_gradient}
In the special case of full gradients, \cite{rammal2023communication} proved a complexity of $\mathcal{O}(\nicefrac{1+\sqrt{c\delta}}{\alpha \varepsilon^2})$, where $\delta=\nicefrac{f}{n}$ and $c$ are two parameters characterizing the agnostic robust aggregator (ARAgg) \cite{KHJ21}.
Note that an $(f,\kappa)$-robust aggregation rule is also a $(\delta,c)$-ARAgg with $c=\nicefrac{\kappa n}{2f}$ \cite{AFGGPS23}.
Therefore, Corollary \ref{coro:full_gradient} slightly improves the complexity result $\mathcal{O}(\nicefrac{1+\sqrt{\kappa}}{\alpha \varepsilon^2})$ from \cite{rammal2023communication}. Furthermore, in the absence of Byzantine faults (i.e., when $\kappa=0$), the iteration complexity simplifies to $T=\mathcal{O}(\nicefrac{\tilde{L}}{\alpha \varepsilon^2})$, which is optimal as shown by \cite{huang2022lower}.
\end{remark}





\section{Experimental Evaluation}


In this section, we evaluate the practical significance of \algname{Byz-EF21-SGDM} through a comprehensive series of comparisons against the state-of-the-art, benchmarking its performance on both binary classification and image classification tasks under four distinct Byzantine threats.

\subsection{Experiment setup} Our algorithm is compared with \algname{SGD} with unbiased compression and robust aggregation (denoted as \algname{BR-CSGD}), \algname{BR-DIANA} \cite{zhu2021broadcast}\footnote{\algname{BR-DIANA} is a simplified version of \algname{BROADCAST} without variance reduction. We do not compare with \algname{BROADCAST} because it consumes a large amount of memory that scales linearly with the number of data points.}, and \algname{Byz-VR-MARINA} \cite{gorbunov2022variance}. For the unbiased contractive compressor in our algorithm, we use $\text{Top}_k$. For all the other algorithms, we use $\text{Rand}_k$, which is ensured to be unbiased and aligns with their theoretical results.

For all methods, the step-size remains fixed throughout the training.  We do not employ learning rate warm-up or decay. Each experiment is run with three different random seeds, and we report the averages of performance across these runs with one standard error.



\subsubsection{Distributed system and datasets}
A distributed system of $n = 20$ workers is considered. Two benchmark datasets are used in the experiments. For both cases, the dataset is divided into 20 equal parts, with each part assigned to one of the 20 clients.

\begin{itemize}
    \item  \textit{Binary classification}: We employ the a9a dataset from LIBSVM \cite{Chang2011libsvm}, a widely recognized benchmark dataset in the field. This dataset comprises 32,561 training instances and 16,281 testing instances. On a9a, we train an $l_2$-regularized logistic regression model. 

    \item \textit{Image classification}: We sample 5\% of the images from the original FEMNIST dataset \cite{CDWLJ+18}, which includes 62 classes of handwritten characters, containing 805,263 training samples and 77,483 testing samples. 
    On FEMNIST, we train a convolutional neural network (CNN)\cite{krizhevsky2012imagenet} with two convolutional layers.

 \end{itemize}


\begin{figure*}[t]

\centering
\includegraphics[width=.8\textwidth]{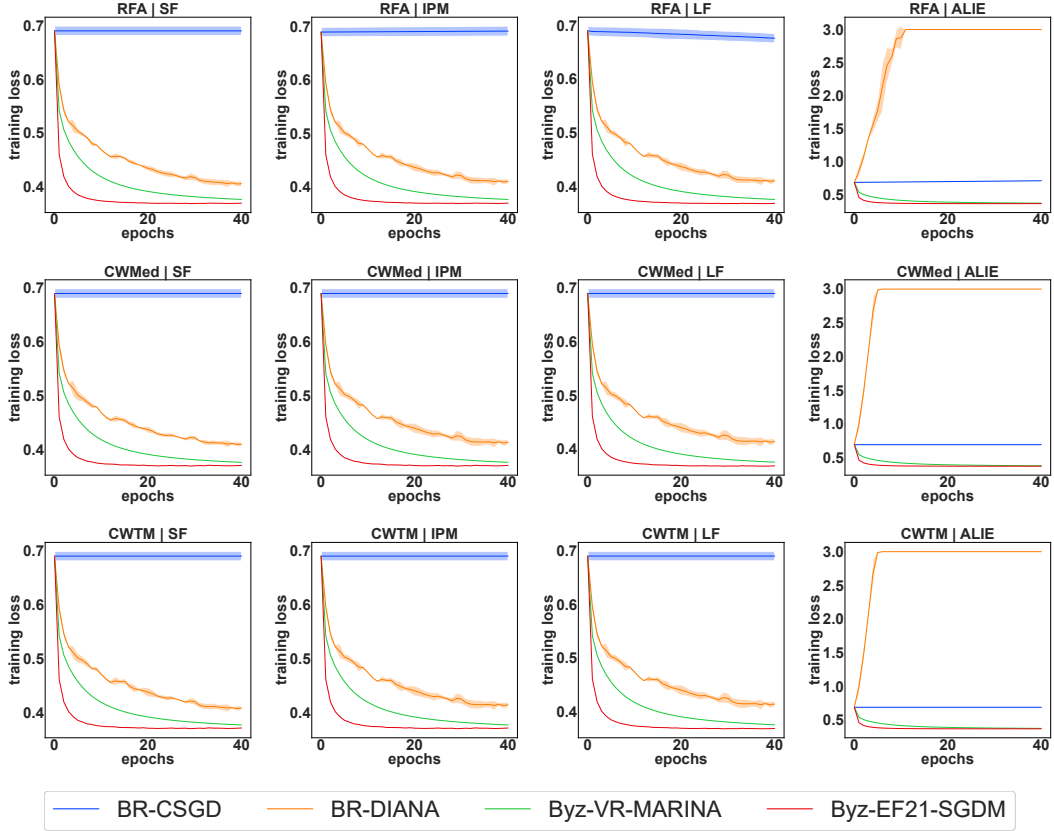}
\caption{The training loss of 3 aggregation rules (\algname{RFA}, \algname{CWMed}, \algname{CWTM}) under 4 attacks (SF, IPM, LF, ALIE) on the a9a dataset. The dataset is uniformly split among 20 workers, including 9 Byzantine workers. \algname{BR-CSGD}, \algname{BR-DIANA}, and \algname{Byz-VR-MARINA} use the $\text{Rand}_1$ compressor. Our algorithm (\algname{Byz-EF21-SGDM}) uses the $\text{Top}_1$ compressor.}
\label{image1}
\end{figure*}

\subsubsection{Robust aggregation rules} 
For our algorithms and the compared algorithms, we use standard averaging (\algname{AVG}) and the following three robust aggregation rules.
\begin{itemize}
	\item \textit{Robust federated averaging (\algname{RFA})} \cite{PKH22}: \algname{RFA} is also known as the geometric median. It is the ``center point” of a given set of vectors, characterized by the smallest sum of distances to all points in the set. In the experiments, we use the smoothed Weiszfeld algorithm with 8 iterations to approximate this point \cite{PKH22}.
    \item \textit{Coordinate-wise median (\algname{CWMed})} \cite{yin2018byzantine}: \algname{CWMed} is obtained by calculating the median for each coordinate axis separately.
    \item  \textit{Coordinate-wise trimmed mean (\algname{CWTM})} \cite{yin2018byzantine}: \algname{CWTM} reduces the impact of noise and outliers by pruning extreme values in each coordinate direction and then calculating the mean of the remaining data.
\end{itemize}

To enhance robustness, they are used in conjunction with the pre-aggregation strategy \algname{NNM} \cite{AFGGPS23}. \algname{NNM} calculates pairwise distances between worker updates, identifies the nearest neighbors for each worker, and computes mixed updates as weighted averages of these neighbors.

\subsubsection{Byzantine attacks}
For both $l_2$-regularized logistic regression and CNN, we set $f = 9$ for both models. The attackers follow 
four powerful strategies: 
\begin{itemize}
    \item \textit{Sign flipping (SF)} \cite{LXCGL19}: Each Byzantine worker $i$ sends $-c_i^{(t+1)}$ instead of $c_i^{(t+1)}$ to the server. 
    \item \textit{Label flipping (LF)} \cite{tolpegin2020data}: The label for each data point within Byzantine workers is flipped. Each Byzantine worker $i$ sends $c_i^{(t+1)}$ to the server, where all the updates are calculated using flipped labels.
    \item \textit{Inner product manipulation (IPM)} \cite{XKG20}: Each attacker $i$ calculates $-{\epsilon}
\lvert \mathcal{H}\rvert^{-1} (\sum_{i\in\mathcal{H}}c_i^{(t+1)})$, {performs a $\text{Top}_k$ operation}, and sends the result to the server. We use $\epsilon = 0.1$ in all experiments. 
    \item \textit{A little is enough (ALIE)} \cite{NEURIPS2019_ec1c5914}: Each attacker $i$ estimates the mean $\mu_\mathcal{H}$ and the coordinate-wise standard deviation $\sigma_\mathcal{H}$ of $\{c_i^{(t+1)}\}_{i\in \mathcal{H}}$. It then computes $\mu_\mathcal{H} - z\sigma_\mathcal{H}$, {performs a $\text{Top}_k$ operation}, and sends the result to the server. Here, $z$ is a small constant that makes the attacks effective but closer to the mean than a significant fraction of honest updates in each coordinate.

\end{itemize}
Upon receiving updates from Byzantine and honest workers, the server adds $g_i^{(t)}$ to the received updates to attain $g_i^{(t+1)}$, for all $i\in[n]$.

\subsection{Empirical results on logistic regression}
For this task, we consider solving a logistic regression problem with an $l_2$-regularization:
\begin{equation*}
    \ell(x, \xi_i) = \log\left( 1 + \exp\left( - b_i a_i x \right) \right) + \lambda \lVert  x\rVert^2
\end{equation*}
where $\xi_i = (a_i, b_i) \in \mathbb{R}^{1 \times d} \times \{ -1, 1 \}$ denotes each data point and $\lambda>0$ denotes the regularization parameter.
We use $\lambda= 1/m$ in this experiments, where $m$ is the number of samples in the local datasets. 
For all methods, we use a batch size of $1$, and select the step-size from the following candidates: $\gamma \in \{0.1, 0.01, 0.001\}$. We use $k=1$ for both $\text{Top}_k$ and $\text{Rand}_k$ compressors. 
Specific parameters for the different algorithms are reported as follows. 
For our algorithm \algname{Byz-EF21-SGDM}, we set the momentum parameter $\eta = 0.01$. For \algname{Byz-VR-MARINA}, we set the probability to compute full gradient as 1 over the number of batches, as suggested in \cite{gorbunov2022variance}. For \algname{BROADCAST}, we set the compressed difference parameter $\beta = 0.01$, which is a typical choice for the \algname{DIANA} framework \cite{horvath2023stochastic}. Finally, the number of epochs is set to 40.

We present the training loss of \algname{BR-CSGD}, \algname{BR-DIANA}, \algname{Byz-VR-MARINA}, and \algname{Byz-EF21-SGDM} on the a9a dataset in Figure \ref{image1}. The findings indicate that our algorithm, \algname{Byz-EF21-SGDM}, and \algname{Byz-VR-MARINA} demonstrate greater robustness against adversarial attacks compared to \algname{BR-CSGD} and \algname{BR-DIANA}. Furthermore, \algname{Byz-EF21-SGDM} achieves faster convergence than \algname{Byz-VR-MARINA} under all considered scenarios.

\subsection{Empirical results on CNN}
We also compare the performance of \algname{BR-CSGD}, \algname{BR-DIANA}, \algname{Byz-VR-MARINA} and \algname{Byz-EF21-SGDM} on training a CNN with two convolutional layers using the FEMNIST dataset. For all methods, we use a batch size of $32$ and select the step-size from the following candidates: $\gamma \in \{0.1, 0.01, 0.001\}$. We use $k=0.1d$ for both $\text{Top}_k$ and $\text{Rand}_k$ compressors. 
For our algorithm \algname{Byz-EF21-SGDM}, the momentum parameter is set as $\eta = 0.1$. For \algname{Byz-VR-MARINA}, we set the probability to compute full gradient as 1 over the number of batches. For \algname{BROADCAST}, we set the compressed difference parameter $\beta = 0.01$. The training process is carried out over 100 epochs.

\begin{figure}[t]

\centering
\includegraphics[width=.45\textwidth]{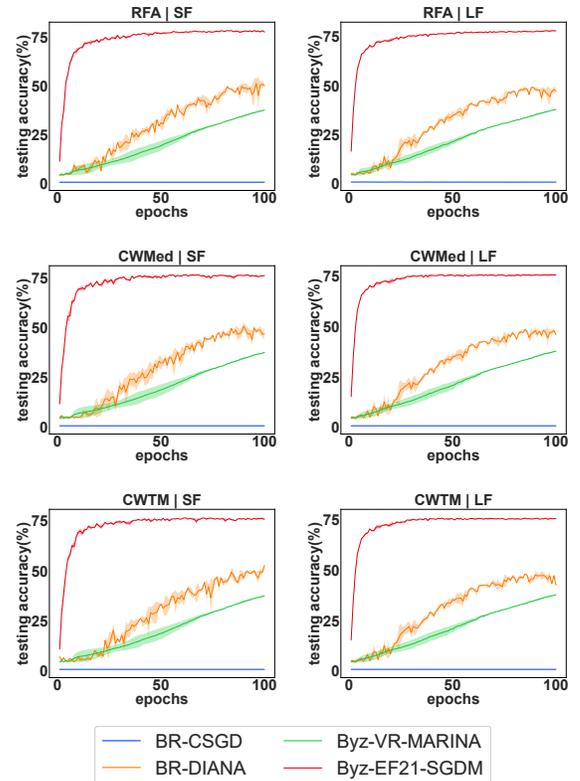}
\caption{The testing accuracy of 3 aggregation rules (\algname{RFA}, \algname{CWMed}, \algname{CWTM}) under 2 attacks (SF, LF) on the FEMNIST dataset. The dataset is uniformly split among 20 workers, including 9 Byzantine workers. \algname{BR-CSGD}, \algname{BR-DIANA}, and \algname{Byz-VR-MARINA} use the $\text{Rand}_{k}$ compressor, and our algorithm (\algname{Byz-EF21-SGDM}) uses the $\text{Top}_{k}$ compressor, where $k=0.1d$.}
\label{image2}
\end{figure}

Figure \ref{image2} presents the testing accuracy of the four methods. The results highlight that under both SF and LF attacks, \algname{Byz-EF21-SGDM} with every choice of robust aggregation rules exhibits significant advantages over all other Byzantine robustness algorithms. 

To compare the performance of \algname{Byz-EF21-SGDM} with different robust aggregation rules, we present Table \ref{table2}. The results suggest that all three robust aggregation rules are effective. Moreover, as additional blocks, they do not compromise performance in the absence of attacks. Finally, \algname{RFA} slightly outperforms the other two methods against all four attacks.

\begin{table}[ht]
    \centering
    \caption{Performance of testing accuracy for our algorithm \algname{Byz-EF21-SGDM} on the FEMNIST dataset  at an adversarial rate of 0.45. N.A. denotes the case with no Byzantine attackers.}
    \begin{tabular}{c|cccccc}
         Aggregation & SF & IPM & LF & ALIE & N.A.  \\
         \hline
         \algname{RFA} +  \algname{NNM} & $\textbf{77.53}$ & $\textbf{77.80}$ & $\textbf{78.00}$ & $\textbf{70.84}$ & $\textbf{80.41}$   \\
         \algname{CWMed} + \algname{NNM} & $76.09$ & $74.87$ & $75.62$ & $70.29$ & $80.13$  \\
         \algname{CWTM} + \algname{NNM} & $75.81$ & $74.65$ & $75.52$ & $70.33$ & $80.30$  \\
         \hline
    \end{tabular}
    \label{table2}
\end{table}



\section{Conclusion}
We have proposed a Byzantine-robust stochastic distributed learning method under communication compression, enhanced with error feedback. The new algorithm is batch-free and guaranteed to converge to a smaller neighborhood around the optimal solution than existing competitors.
We have proven the convergence rate and robustness guarantees for the proposed algorithm. Additionally, we have demonstrated the advantages of the method through experimental studies on both binary classification with convex logistic loss and image classification with non-convex loss in a heterogeneous setting.

This work establishes a foundation for numerous future research endeavors, particularly in the realm of decentralized problem setups.


\appendices

\section{Proof of Theorem \ref{thm:rate}}\label{appdxSec:proofThm}


We begin by presenting four technical lemmas that establish the foundation for proving Theorem \ref{thm:rate}. Following this, we proceed with the proof of Theorem \ref{thm:rate}.

\subsection{Key lemmas}

We now state the following lemma, which is instrumental in the analysis of non-convex optimization methods \cite{li2021page}.

\begin{lemma}[Descent lemma]\label{lem:descent}
	Given an $L$-smooth function $\mathcal{L}$. For the update $x^{(t+1)}= x^{(t)}-\gamma g^{(t)}$, there holds
	\begin{equation*}
		\begin{split}
		\mathcal{L}(x^{(t+1)}) \leq & \mathcal{L}(x^{(t)}) - \frac{\gamma}{2} \lVert  \nabla \mathcal{L}(x^{(t)}) \rVert^2+ \frac{\gamma}{2}\lVert g^{(t)} - \nabla \mathcal{L}(x^{(t)}) \rVert^2 \\
		&- (\frac{1}{2\gamma}- \frac{L}{2}) \lVert 
		x^{(t+1)} - x^{(t)}\rVert^2. 
		\end{split}
	\end{equation*}
\end{lemma}

\begin{lemma}[Robust aggregation error]\label{lem:robust_aggregation}
	Suppose Assumption \ref{assump:bounded_hetero} holds.
	Then for all $t\geq 0$ the iterates produced by \algname{Byz-EF21-SGDM} in Algorithm \ref{alg:Byz-EF21-SGDM} satisfy
	\begin{equation*}
		\begin{split}
			\lVert g^{(t)} - \overline{g}^{(t)} \rVert^2 
			\leq & \frac{6\kappa(n-f-1)}{(n-f)^2} \sum_{i\in \mathcal{H}} \left( \lVert C_i^{(t)} \rVert^2 + \lVert M_i^{(t)} \rVert^2   \right) \\
			& + \frac{6\kappa(n-f-1)}{n-f} G^2
		\end{split}
	\end{equation*}
	where $\overline{g}^{(t)} := {(n-f)^{-1}}\sum_{i\in \mathcal{H}}g_i$, $C_i^{(t)}= {g}_i^{(t)} -  {v}_i^{(t)} $ and $M_i^{(t)}=  {v}_i^{(t)} - \nabla \mathcal{L}_i(x^{(t)})$.
\end{lemma}
\begin{proof}
	Define $H_i^{(t)}=  \nabla \mathcal{L}_i(x^{(t)})-\nabla \mathcal{L}(x^{(t)})$. We consider
	\begin{equation*}
		\begin{split}
			&\sum_{i,j\in \mathcal{H}} \lVert g_i^{(t)}- g_j^{(t)} \rVert^2  \\
			& = \sum_{i,j\in \mathcal{H}, i\neq j} \lVert C_i^{(t)}+ M_i^{(t)} + H_i^{(t)}- C_j^{(t)}- M_j^{(t)} - H_j^{(t)} \rVert^2 \\
			& \leq 6  \sum_{i,j\in \mathcal{H}, i\neq j} 
			\Big( \lVert C_i^{(t)} \rVert^2 + \lVert M_i^{(t)} \rVert^2 + \lVert H_i^{(t)} \rVert^2 + \lVert C_j^{(t)} \rVert^2  \\
			& \quad \quad\quad\quad\quad\quad + \lVert M_j^{(t)} \rVert^2 + \lVert H_j^{(t)} \rVert^2 \Big) \\
			& = 12(n-f-1)  \sum_{i\in \mathcal{H}} \left( \lVert C_i^{(t)} \rVert^2 + \lVert M_i^{(t)} \rVert^2 + \lVert H_i^{(t)} \rVert^2  \right) \\
			& \leq 12(n-f-1)  \sum_{i\in \mathcal{H}} \left( \lVert C_i^{(t)} \rVert^2 + \lVert M_i^{(t)} \rVert^2   \right)  \\
			 & \quad +  12(n-f-1)(n-f) G^2.
		\end{split}
	\end{equation*}
	Because 
	\begin{equation*}
		\sum_{i\in \mathcal{H}} \lVert g_i^{(t)}- \overline{g}^{(t)} \rVert^2 = \frac{1}{2(n-f)} \sum_{i,j\in \mathcal{H}} \lVert g_i^{(t)}- g_j^{(t)} \rVert^2
	\end{equation*}
	and
	the aggregation rule is $(f,\kappa)$-robust, we have
	\begin{equation*}
		\begin{split}
			\lVert g^{(t)} - \overline{g}^{(t)} \rVert^2  
			&\leq \frac{\kappa}{n-f}\sum_{i\in \mathcal{H}} \lVert g_i^{(t)}- \overline{g}^{(t)} \rVert^2 \\
			& \leq  \frac{\kappa}{2(n-f)^2}\sum_{i,j\in \mathcal{H}} \lVert g_i^{(t)}- g_j^{(t)} \rVert^2 \\
			& \leq \frac{6\kappa(n-f-1)}{(n-f)^2} \sum_{i\in \mathcal{H}} \left( \lVert C_i^{(t)} \rVert^2 + \lVert M_i^{(t)} \rVert^2   \right) \\
			& \quad + \frac{6\kappa(n-f-1)}{n-f} G^2.
		\end{split}
	\end{equation*}
\end{proof}

\begin{lemma}[Accumulated compression error]\label{lem:compression_error}
	Suppose Assumptions \ref{assump:smoothness} and \ref{assump:bounded_var} hold. Then for all $t\geq 0$ the iterates produced by \algname{Byz-EF21-SGDM} in Algorithm \ref{alg:Byz-EF21-SGDM} satisfy                       
	\begin{equation*}
		\begin{split}
			\sum_{t=0}^{T-1}\mathbb{E}\left[ \lVert C_i^{(t)} \rVert^2 \right]  \leq &  \frac{8\eta^2}{\alpha^2} \sum_{t=0}^{T-1} \mathbb{E}\left[\lVert   M_i^{(t)}\rVert^2 \right]    + \frac{2 (1-\alpha)  \eta^2 T \sigma^2}{\alpha}     \\
			&+ \frac{8\eta^2L_i^2}{\alpha^2} \sum_{t=0}^{T-1} \mathbb{E}\left[\lVert x^{(t+1)}-x^{(t)}\rVert^2 \right], \forall i \in\mathcal{H}
		\end{split}
	\end{equation*}
	where $C_i^{(t)}= {g}_i^{(t)} -  {v}_i^{(t)} $ and $M_i^{(t)}=  {v}_i^{(t)} - \nabla \mathcal{L}_i(x^{(t)})$.
\end{lemma}

\begin{proof}
	Recall the update
	\begin{equation*}
		\begin{split}
			g_i^{(t)}  = g_i^{(t-1)} + \mathcal{C}(v_i^{t}-g_i^{(t-1)}).
		\end{split}
	\end{equation*}
	For $i\in\mathcal{H}$, there holds
	\begin{equation*}
		\begin{split}
			&\mathbb{E}\left[ \lVert C_i^{(t)}\rVert^2 \right] \\
			&=    \mathbb{E}\left[ \lVert g_i^{(t-1)} - v_i^{(t)} + \mathcal{C}(v_i^{t}-g_i^{(t-1)}) \rVert^2 \right]  \\
			& =  \mathbb{E} \left[ \mathbb{E}_{\mathcal{C}} \left[ \lVert v_i^{(t)}- g_i^{(t-1)}  - \mathcal{C}(v_i^{t}-g_i^{(t-1)}) \rVert^2 \right] \right]       \\
			& \overset{(i)}{\leq} (1-\alpha)  \mathbb{E}\left[ \lVert v_i^{(t) }-g_i^{(t-1)} \rVert^2 \right] \\
			&  \overset{(ii)}{=} (1-\alpha)  \mathbb{E}\Big[ \mathbb{E}_{\xi_i^{(t)}}\Big[\lVert  v_i^{(t-1)} -g_i^{(t-1)} \\
			& \quad \quad \quad \quad \quad \quad \quad \quad \,\, + \eta (\nabla_x \ell_i(x^{(t)},\xi_i^{(t)})-v_i^{(t-1)}) \rVert^2 \Big] \Big] \\
			& = (1-\alpha)  \mathbb{E}\left[\lVert  v_i^{(t-1)} -g_i^{(t-1)} + \eta (\nabla \mathcal{L}_i(x^{(t)})-v_i^{(t-1)}) \rVert^2 \right]  \\
			& \quad  +(1-\alpha)  \eta^2 \mathbb{E}\left[\lVert  \nabla_x \ell_i(x^{(t)},\xi_i^{(t)})- \nabla \mathcal{L}_i(x^{(t)})\rVert^2 \right] \\
			& \overset{(iii)}{=} (1-\alpha)(1+\rho)  \mathbb{E}\left[\lVert  C_i^{(t-1)} \rVert^2 \right]  +(1-\alpha)  \eta^2 \sigma^2\\
			& \quad+ (1-\alpha)(1+\rho^{-1}) \eta^2\mathbb{E}\left[\lVert   \nabla \mathcal{L}_i(x^{(t)})-v_i^{(t-1)}\rVert^2 \right]     \\
			& \leq (1-\alpha)(1+\rho)  \mathbb{E}\left[\lVert  C_i^{(t-1)} \rVert^2 \right] +(1-\alpha)  \eta^2 \sigma^2 \\
			& \quad+ 2(1-\alpha)(1+\rho^{-1}) \eta^2\mathbb{E}\left[\lVert   M_i^{(t-1)}\rVert^2 \right]   \\
			&  \quad+ 2(1-\alpha)(1+\rho^{-1}) \eta^2 L_i^2 \mathbb{E}\left[\lVert x^{(t)}-x^{(t-1)}\rVert^2 \right]   . 
		\end{split}
	\end{equation*}
	where $(i)$ uses the contractive property as defined in Definition \ref{def:contractive}, $(ii)$ is due to the update
	\begin{equation*}
		v_i^{(t)}  = v_i^{(t-1)} + \eta (\nabla_x \ell_i(x^{(t)},\xi_i^{(t)})-v_i^{(t-1)}),
	\end{equation*}
	and $(iii)$ holds by Assumption \ref{assump:bounded_var} and Young's inequality for any $\rho>0$.
	Setting $\rho=\alpha/2$, we obtain
	\begin{equation*}
		(1-\alpha)(1+\rho) = 1-\frac{\alpha}{2}- \frac{\alpha^2}{2}\leq 1-\frac{\alpha}{2}
	\end{equation*}
	and 
	\begin{equation*}
		(1-\alpha)(1+\rho^{-1}) = \frac{2}{\alpha}-\alpha -1 \leq \frac{2}{\alpha}.
	\end{equation*}
	Thus, there holds
	\begin{equation*}
		\begin{split}
			\mathbb{E}\left[ \lVert C_i^{(t)}\rVert^2 \right] 
			 \leq & \left(1-\frac{\alpha}{2}\right)  \mathbb{E}\left[\lVert  C_i^{(t-1)} \rVert^2 \right] +  \frac{4\eta^2}{\alpha}\mathbb{E}\left[\lVert   M_i^{(t-1)}\rVert^2 \right]  \\
			&  + \frac{4\eta^2L_i^2}{\alpha} \mathbb{E}\left[\lVert x^{(t)}-x^{(t-1)}\rVert^2 \right]   +(1-\alpha)  \eta^2 \sigma^2. 
		\end{split}
	\end{equation*}
	Summing up the above inequality from $t=0$ to $t=T-1$ leads to
	\begin{equation*}
		\begin{split}
			\sum_{t=0}^{T-1}\mathbb{E}\left[ \lVert C_i^{(t)}\rVert^2 \right] 
			\leq & \frac{8\eta^2}{\alpha^2} \sum_{t=0}^{T-1} \mathbb{E}\left[\lVert   M_i^{(t)}\rVert^2 \right]    + \frac{2 (1-\alpha)  \eta^2 T \sigma^2}{\alpha}  \\
			&    + \frac{8\eta^2L_i^2}{\alpha^2} \sum_{t=0}^{T-1} \mathbb{E}\left[\lVert x^{(t+1)}-x^{(t)}\rVert^2 \right].
		\end{split}
	\end{equation*}

\end{proof}


\begin{lemma}[Accumulated momentum deviation]\label{lem:momentum}
	Suppose Assumptions \ref{assump:smoothness} and \ref{assump:bounded_var} hold. Then for all $t\geq 0$ the iterates produced by \algname{Byz-EF21-SGDM} in Algorithm \ref{alg:Byz-EF21-SGDM} satisfy     
	\begin{equation*}
		\begin{split}
		&\frac{1}{n-f} \sum_{t=0}^{T-1} \sum_{i\in\mathcal{H}} \mathbb{E}[\lVert M_i^{(t)}  \rVert^2] \\
			& \leq \frac{\tilde{L}^2}{\eta^2}\sum_{t=0}^{T-1}\mathbb{E} \left[  \lVert x^{(t+1)}-x^{(t)}  \rVert^2 \right] + \eta T \sigma^2 \\
			& \quad + \frac{1}{\eta (n-f)} \sum_{i\in\mathcal{H}} \mathbb{E} \left[  \lVert v_i^{(0)}-\nabla \mathcal{L}_i(x^{(0)})  \rVert^2 \right]
		\end{split}
	\end{equation*}
	and
	\begin{equation*}
		\begin{split}
		&	\sum_{t=0}^{T-1}  \mathbb{E}[\lVert \overline{v}^{(t)} - \nabla \mathcal{L}_{\mathcal{H}}(x^{(t)}) \rVert^2] \\
			& \leq  \frac{L^2}{\eta^2}\sum_{t=0}^{T-1}\mathbb{E} \left[  \lVert x^{(t+1)}-x^{(t)}  \rVert^2 \right] + \frac{\eta T \sigma^2}{n-f} \\
			& \quad + \frac{1}{\eta } \mathbb{E} \left[  \lVert \overline{v}^{(0)}-\nabla \mathcal{L}_{\mathcal{H}}(x^{(0)})  \rVert^2 \right]    
		\end{split}
	\end{equation*} 
	where $\overline{v}^{(t)}= (n-f)^{-1} \sum_{i\in\mathcal{H}}v_i^{(t)}$.
\end{lemma}

\begin{proof}
	Recall
	\begin{equation*}
		v_i^{(t)}=(1-\eta) v_i^{(t-1)} +  \eta \nabla_x \ell_i(x^{(t)},\xi_i^{(t)})
	\end{equation*}
	and consider
	\begin{equation*}
		\begin{split}
			\lVert M_i^{(t)} \rVert^2  = & \lVert  (1-\eta)(v_i^{(t-1)}-\nabla \mathcal{L}_i(x^{(t)}))  \\
			& \quad + \eta (\nabla_x \ell_i(x^{(t)}_i, \xi_i^{(t)}) - \nabla \mathcal{L}_i(x^{(t)})) \rVert^2. 
		\end{split}
	\end{equation*}
	Taking expectation on both sides and using the law of total expectation, we obtain
	\begin{equation*}
		\begin{split}
			\mathbb{E}[\lVert M_i^{(t)} \rVert^2]  =  & \mathbb{E} \Big[ \mathbb{E}_{\xi_i^{(t)}} \Big[\lVert  (1-\eta)(v_i^{(t-1)}-\nabla \mathcal{L}_i(x^{(t)}))  \\
			&  \quad \quad \quad \,\,+ \eta (\nabla_x \ell_i(x^{(t)}_i, \xi_i^{(t)}) - \nabla \mathcal{L}_i(x^{(t)})) \rVert^2 \Big]\Big]   
		\end{split}
	\end{equation*}
	Because of 
	\begin{equation*}
		\mathbb{E}_{\xi_i^{(t)}} \left[\nabla_x \ell_i(x^{(t)}_i, \xi_i^{(t)}) - \nabla \mathcal{L}_i(x^{(t)}) \right] = 0,
	\end{equation*}
	there holds
	\begin{equation*}
		\begin{split}
		&	\mathbb{E}[\lVert M_i^{(t)} \rVert^2] \\
			& =  (1-\eta)^2\mathbb{E} \left[  \lVert v_i^{(t-1)}-\nabla \mathcal{L}_i(x^{(t)})  \rVert^2 \right]  \\
			& \quad  + \eta^2 \mathbb{E}   \left[\lVert  \nabla_x \ell_i(x^{(t)}_i, \xi_i^{(t)}) - \nabla \mathcal{L}_i(x^{(t)}) \rVert^2 \right]  \\
			& \leq  (1-\eta)^2 (1+a)\mathbb{E} \left[  \lVert M_i^{(t-1)}  \rVert^2 \right]+ \eta^2 \sigma^2 \\
			& \quad   + (1-\eta)^2 \left(1+{a}^{-1} \right)\mathbb{E} \left[  \lVert \nabla \mathcal{L}_i(x^{(t-1)})-\nabla \mathcal{L}_i(x^{(t)})  \rVert^2 \right] .
		\end{split}
	\end{equation*}
	for any $a>0$.
	We take $a=\eta(1-\eta)^{-1}$ and use the $L_i$-smoothness of $\mathcal{L}_i$ to obtain
	\begin{equation*}
		\begin{split}
		&	\mathbb{E}[\lVert M_i^{(t)}  \rVert^2] \\
			&\leq  (1-\eta)\mathbb{E} \left[  \lVert M_i^{(t-1)}  \rVert^2 \right]   + \frac{ L_i^2}{\eta} \mathbb{E} \left[  \lVert x^{(t-1)}-x^{(t)}  \rVert^2 \right] + \eta^2 \sigma^2.
		\end{split}
	\end{equation*}
	Summing the above inequality over all $i\in\mathcal{H}$ and from $t=0$ to $t=T-1$ yields
	\begin{equation*}
		\begin{split}
			&\frac{1}{n-f} \sum_{t=0}^{T-1} \sum_{i\in\mathcal{H}} \mathbb{E}[\lVert M_i^{(t)}  \rVert^2]  \\
		&	\leq \frac{\tilde{L}^2}{\eta^2}\sum_{t=0}^{T-1}\mathbb{E} \left[  \lVert x^{(t+1)}-x^{(t)}  \rVert^2 \right]+ \eta T \sigma^2  \\
		& \quad + \frac{1}{\eta (n-f)} \sum_{i\in\mathcal{H}} \mathbb{E} \left[  \lVert v_i^{(0)}-\nabla \mathcal{L}_i(x^{(0)})  \rVert^2 \right].      
		\end{split}
	\end{equation*}

	Using the same arguments, we obtain
	\begin{equation*}
		\begin{split}
		&	\mathbb{E}[\lVert \overline{v}^{(t)} - \nabla \mathcal{L}_{\mathcal{H}}(x^{(t)}) \rVert^2] \\
			&\leq  (1-\eta)\mathbb{E} \left[  \lVert \overline{v}^{(t-1)}-\nabla \mathcal{L}_{\mathcal{H}}(x^{(t-1)})  \rVert^2 \right]  \\
			& \quad   + \frac{ L^2}{\eta} \mathbb{E} \left[  \lVert x^{(t-1)}-x^{(t)}  \rVert^2 \right] + \frac{\eta^2 \sigma^2}{n-f}
		\end{split}
	\end{equation*}
	and
	\begin{equation*}
		\begin{split}
			&\sum_{t=0}^{T-1}\mathbb{E}[\lVert \overline{v}^{(t)} - \nabla \mathcal{L}_{\mathcal{H}}(x^{(t)}) \rVert^2] \\
			& \leq \frac{L^2}{\eta^2}\sum_{t=0}^{T-1}\mathbb{E} \left[  \lVert x^{(t+1)}-x^{(t)}  \rVert^2 \right]+ \frac{\eta T \sigma^2}{n-f}  \\
			& \quad + \frac{1}{\eta } \mathbb{E} \left[  \lVert \overline{v}^{(0)}-\nabla \mathcal{L}_{\mathcal{H}}(x^{(0)})  \rVert^2 \right].      
		\end{split}
	\end{equation*}
\end{proof}

\subsection{Proof of Theorem \ref{thm:rate}}
\begin{proof}
	By Lemma \ref{lem:descent}, there holds, for any $\gamma \leq 1/(2L)$,
	\begin{equation}\label{eq:descent}
		\begin{split}
		\mathcal{L}_{\mathcal{H}}(x^{(t+1)}) \leq  & \mathcal{L}_{\mathcal{H}}(x^{(t)}) - \frac{\gamma}{2} \lVert  \nabla \mathcal{L}(x^{(t)}) \rVert^2 - \frac{1}{4\gamma} \lVert 
		x^{(t+1)} - x^{(t)}\rVert^2 \\
		& + \frac{\gamma}{2}\lVert g^{(t)} - \nabla \mathcal{L}_{\mathcal{H}}(x^{(t)}) \rVert^2. 
	\end{split}
	\end{equation} 
	Summing the above from $t=0$ to $T-1$ and taking expectation, we have
	\begin{equation}\label{proof:desc}
		\begin{split}
		&	\frac{1}{T} \sum_{t=0}^{T-1} \mathbb{E} \left[ \lVert \nabla \mathcal{L}_{\mathcal{H}}(x^{(t)}) \rVert^2 \right] \\
			& \leq  \frac{2\delta_0 }{\gamma T}- \frac{1}{2\gamma^2 T}  \sum_{t=0}^{T-1} \mathbb{E}\left[ \lVert 
			x^{(t+1)} - x^{(t)}\rVert^2  \right] \\
			& \quad +  \frac{1}{T}\sum_{t=0}^{T-1} \mathbb{E} \left[\lVert g^{(t)} - \nabla \mathcal{L}_{\mathcal{H}}(x^{(t)}) \rVert^2 \right]
		\end{split}
	\end{equation}
	where $\delta_t = \mathbb{E} \left[ \mathcal{L}_{\mathcal{H}}(x^{(t)})-\mathcal{L}_{\mathcal{H}}^* \right]$. We note that
	\begin{equation}\label{eq:bound_Byzantine_error}
		\begin{split}
			& \lVert g^{(t)} - \nabla \mathcal{L}_{\mathcal{H}}(x^{(t)}) \rVert^2  \\
			& \leq  3\lVert g^{(t)}-\overline{g}^{(t)} \rVert^2 +  3\lVert \overline{g}^{(t)} -  \overline{v}^{(t)}\rVert^2 + 3\lVert  \overline{v}^{(t)} - \nabla \mathcal{L}_{\mathcal{H}}(x^{(t)}) \rVert^2 \\
			& \leq {3\lVert g^{(t)}-\overline{g}^{(t)} \rVert^2} + {\frac{3}{n-f} \sum_{i\in\mathcal{H}}\lVert {g}_i^{(t)} -  {v}_i^{(t)}\rVert^2}  \\
			& \quad + {3\lVert  \overline{v}^{(t)} - \nabla \mathcal{L}_{\mathcal{H}}(x^{(t)}) \rVert^2}
		\end{split}
	\end{equation}
	where $\overline{g}= (n-f)^{-1} \sum_{i\in\mathcal{H}} g_i$ and $\overline{v}= (n-f)^{-1} \sum_{i\in\mathcal{H}}v_i$.

	Next, we use the technical lemmas in the previous section to bound the deviation between $g^{(t)}$ and $\nabla \mathcal{L}_{\mathcal{H}}(x^{(t)})$. First, we use Lemma \ref{lem:robust_aggregation} to obtain
	\begin{equation*}
		\begin{split}
		&	\lVert g^{(t)} - \nabla \mathcal{L}_{\mathcal{H}}(x^{(t)}) \rVert^2  \\
			& \leq  \frac{18\kappa(n-f-1)}{(n-f)^2} \sum_{i\in \mathcal{H}} \left( \lVert C_i^{(t)} \rVert^2 + \lVert M_i^{(t)} \rVert^2   \right)   \\
			& \quad + \frac{3}{n-f} \sum_{i\in\mathcal{H}}\lVert C_i^{(t)} \rVert^2  + 3\lVert \widetilde{M}^{(t)} \rVert^2 +  \frac{18\kappa(n-f-1)}{n-f} G^2   \\
			& \leq  \frac{3( 6\kappa +1 )}{n-f} \sum_{i\in\mathcal{H}} \lVert C_i^{(t)} \rVert^2 +   \frac{18\kappa}{n-f} \sum_{i\in \mathcal{H}}  \lVert M_i^{(t)} \rVert^2 + 3\lVert  \widetilde{M}^{(t)} \rVert^2  \\
			& \quad   + {18\kappa} G^2 
		\end{split}
	\end{equation*}
	where $C_i^{(t)}= {g}_i^{(t)} -  {v}_i^{(t)} $, $M_i^{(t)}=  {v}_i^{(t)} - \nabla \mathcal{L}_i(x^{(t)})$, and $\widetilde{M}^{(t)} = \overline{v}^{(t)} - \nabla \mathcal{L}_{\mathcal{H}}(x^{(t)}) $. By summing up the above inequality from $t=0$ to $t=T-1$, we obtain
	\begin{equation*}
		\begin{split}
		&\sum_{t=0}^{T-1} \lVert g^{(t)} - \nabla \mathcal{L}_{\mathcal{H}}(x^{(t)}) \rVert^2  \\
	&	\leq  \frac{ 3(6\kappa +1) }{n-f}   \sum_{t=0}^{T-1}\sum_{i\in\mathcal{H}} \lVert C_i^{(t)} \rVert^2  + \frac{18\kappa}{n-f} 
		\sum_{t=0}^{T-1} \sum_{i\in \mathcal{H}}  \lVert M_i^{(t)} \rVert^2 \\
		& \quad  + 3 \sum_{t=0}^{T-1}\lVert  \widetilde{M}^{(t)} \rVert^2    + 18 \kappa T G^2.
				\end{split}
	\end{equation*}
	Then, by taking expectation and using Lemma \ref{lem:compression_error}, we have
	\begin{equation*}
		\begin{split}
			&\sum_{t=0}^{T-1} \mathbb{E} \left[  \lVert g^{(t)} - \nabla \mathcal{L}_{\mathcal{H}}(x^{(t)}) \rVert^2  \right ]\\
			& \leq  \frac{ 3(6\kappa +1) }{n-f} \sum_{i\in\mathcal{H}}\Big(\frac{8\eta^2}{\alpha^2} \sum_{t=0}^{T-1} \mathbb{E}\left[\lVert   M_i^{(t)}\rVert^2 \right]    + \frac{2 (1-\alpha)  \eta^2 T \sigma^2}{\alpha}     \\
			& \quad \quad \quad   \quad \quad \quad  \quad \quad+ \frac{8\eta^2L_i^2}{\alpha^2} \sum_{t=0}^{T-1} \mathbb{E}\left[\lVert x^{(t+1)}-x^{(t)}\rVert^2 \right] \Big) \\
			& \quad + \frac{18\kappa}{n-f} 
			\sum_{t=0}^{T-1} \sum_{i\in \mathcal{H}} \mathbb{E} \left[   \lVert M_i^{(t)} \rVert^2 \right] + 3 \sum_{t=0}^{T-1} \mathbb{E} \left[ \lVert  \widetilde{M}^{(t)} \rVert^2 \right] \\
			& \quad     + {18 \kappa T} G^2 \\
			& \leq  \left(\frac{24\eta^2(6\kappa+1)}{\alpha^2(n-f)}  + \frac{18\kappa}{n-f}\right) \sum_{t=0}^{T-1} \sum_{i\in \mathcal{H}} \mathbb{E} \left[  \lVert M_i^{(t)} \rVert^2  \right] \\
			& \quad + \frac{24(6\kappa+1)\eta^2\tilde{L}^2}{\alpha^2} \sum_{t=0}^{T-1} \mathbb{E}\left[\lVert x^{(t+1)}-x^{(t)}\rVert^2 \right]  \\
			& \quad + 3 \sum_{t=0}^{T-1} \mathbb{E} \left[ \lVert  \widetilde{M}^{(t)} \rVert^2 \right] +  \frac{6(6\kappa+1)  \eta^2 T \sigma^2}{\alpha} + 18\kappa T G^2 \\
			& \quad  + \frac{24\eta^2(6\kappa+1)}{\alpha^2 (n-f)} \sum_{i\in\mathcal{H}} \mathbb{E} \left[  \lVert M_i^{(0)}  \rVert^2 \right]  
		\end{split}
	\end{equation*}
	Furthermore, by using Lemma \ref{lem:momentum}, we get
	\begin{equation*}
		\scriptsize
		\begin{split}
			&\sum_{t=0}^{T-1} \mathbb{E} \left[ \lVert g^{(t)} - \nabla \mathcal{L}_{\mathcal{H}}(x^{(t)}) \rVert^2   \right] \\
			& \leq  \left(\frac{24\eta^2(6\kappa+1)}{\alpha^2}  + {18\kappa}\right) \biggl(   \frac{\tilde{L}^2}{\eta^2}\sum_{t=0}^{T-1}\mathbb{E} \left[  \lVert x^{(t+1)}-x^{(t)}  \rVert^2 \right]+ \eta T \sigma^2  \\
			& \quad \quad \quad \quad \quad \quad\quad \quad \quad\quad \quad \quad\quad \quad + \frac{1}{\eta (n-f)}  \sum_{i\in\mathcal{H}}\mathbb{E} \left[  \lVert M_i^{(0)}  \rVert^2 \right] \biggl) \\
			& \quad +   \frac{3L^2}{\eta^2}\sum_{t=0}^{T-1}\mathbb{E} \left[  \lVert x^{(t+1)}-x^{(t)}  \rVert^2 \right]+ \frac{3\eta T \sigma^2}{n-f} + \frac{3}{\eta } \mathbb{E} \left[  \lVert \widetilde{M}^{(0)}  \rVert^2 \right] \\
			& \quad + \frac{24\eta^2(6\kappa+1)}{\alpha^2 (n-f)} \sum_{i\in\mathcal{H}} \mathbb{E} \left[  \lVert M_i^{(0)}  \rVert^2 \right]+ \frac{6(6\kappa+1)  \eta^2 T \sigma^2}{\alpha}   \\
			& \quad + \frac{24(6\kappa+1)\eta^2\tilde{L}^2}{\alpha^2} \sum_{t=0}^{T-1} \mathbb{E}\left[\lVert x^{(t+1)}-x^{(t)}\rVert^2 \right] + 18\kappa T G^2   \\
			& \leq  3\tilde{L}^2\left(\frac{8(6\kappa+1)}{\alpha^2} + \frac{6\kappa}{\eta^2} + \frac{L^2}{\eta^2 \tilde{L}^2} + \frac{8(6\kappa+1)\eta^2}{\alpha^2} \right) \sum_{t=0}^{T-1} \mathbb{E}\left[\lVert x^{(t+1)}-x^{(t)}\rVert^2 \right] \\
			& \quad + 3\eta T \left( \frac{8\eta^2 (6\kappa +1) }{\alpha^2} +6\kappa + \frac{1}{n-f} + \frac{2(6\kappa+1)\eta}{\alpha}\right) \sigma^2 + 18\kappa TG^2 \\
			& \quad + \frac{3}{\eta}  \mathbb{E} \left[  \lVert \widetilde{M}^{(0)}  \rVert^2 \right] + \left(\frac{24\eta(6\kappa+1)(1+\eta)}{(n-f)\alpha^2}  + \frac{18\kappa}{\eta(n-f)}\right) \sum_{i\in\mathcal{H}} \mathbb{E} \left[  \lVert M_i^{(0)}  \rVert^2 \right]
		\end{split}
	\end{equation*}
	Plugging the above relation into \eqref{proof:desc} leads to
	\begin{equation*}
		\small
		\begin{split}
		&	\mathbb{E} \left[ \lVert \nabla \mathcal{L}_{\mathcal{H}}(\hat{x} ^{(T)}) \rVert^2 \right]  \\
			& \leq  \frac{2\delta_0 }{\gamma T}- \frac{P_1}{T} \sum_{t=0}^{T-1} \mathbb{E} \left[ \lVert 
			x^{(t+1)} - x^{(t)}\rVert^2 \right] \\
			& \quad + 3\eta  \left( \frac{8\eta^2 (6\kappa +1) }{\alpha^2} +6\kappa + \frac{1}{n-f} + \frac{2(6\kappa+1)\eta}{\alpha}\right) \sigma^2 + 18\kappa G^2 \\
			& \quad  +  \frac{1}{T(n-f)}\left(\frac{24\eta(6\kappa+1)(1+\eta)}{\alpha^2}  + \frac{18\kappa}{\eta}\right) \sum_{i\in\mathcal{H}} \mathbb{E} \left[  \lVert M_i^{(0)}  \rVert^2 \right] \\
			& \quad + \frac{3}{\eta T}  \mathbb{E} \left[  \lVert \widetilde{M}^{(0)}  \rVert^2 \right]
		\end{split}
	\end{equation*}
	where $\hat{x}^{(T)}$ is sampled uniformly at random from $T$ iterates and
	\begin{equation*}
		\begin{split}
			P_1
			&= \frac{1}{\gamma^2 } \biggl( \frac{1}{2}- \frac{24(6\kappa+1)\gamma^2\tilde{L}^2}{\alpha^2} - \frac{18 \kappa \gamma^2 \tilde{L}^2}{\eta^2} \\
			& \quad \quad\quad \quad - \frac{3 \gamma^2L^2}{\eta^2} - \frac{24(6\kappa+1)  \gamma^2 \tilde{L}^2 \eta^2}{\alpha^2}\biggl)       \\
			& \overset{(i)}{\geq}  \frac{1}{\gamma^2 } \left( \frac{1}{2}- \frac{48(6\kappa+1)\gamma^2\tilde{L}^2}{\alpha^2} - \frac{3\gamma^2(6 \kappa \tilde{L}^2+L^2)}{\eta^2}  \right)   \\
			&\overset{(ii)}{\geq} 0
		\end{split}
	\end{equation*}
	where $(i)$ and $(ii)$ are due to $\eta\leq 1$ and the assumption on step-size, respectively. We proved \eqref{eq:bound}.

	Finally, by using the choice of the momentum parameter
	\begin{equation*}
		\begin{split}
		\eta \leq \min\biggl\{ &  \left( \frac{L\delta_0\alpha^2}{24(1+6\kappa )\sigma^2 T} \right)^{\nicefrac{1}{4}}, \left( \frac{L\delta_0\alpha}{6(1+6\kappa)\sigma^2 T} \right)^{\nicefrac{1}{3}}, \\
	&	\left( \frac{L\delta_0(n-f)}{3\sigma^2 T} \right)^{\nicefrac{1}{2}} ,\left( \frac{L\delta_0}{18\kappa\sigma^2 T} \right)^{\nicefrac{1}{2}} 
	\biggl\}
			\end{split}
	\end{equation*}
	we obtain
	\begin{equation*}
		\begin{split}
		&	\mathbb{E} \left[ \lVert \nabla \mathcal{L}_{\mathcal{H}}(\hat{x} ^{(T)}) \rVert^2 \right]  \\
		&	\leq  \left( \frac{(24(1+6\kappa))^{\nicefrac{1}{3}}L\delta_0\sigma^{\nicefrac{2}{3}}}{\alpha^{\nicefrac{2}{3}} T} \right)^{\nicefrac{3}{4}}   + \left( \frac{( 6(1+6\kappa))^{\nicefrac{1}{2}}L\delta_0 \sigma}{\alpha^{\nicefrac{1}{2}}T} \right)^{\nicefrac{2}{3}} \\
		& \quad +  \left( \frac{18\kappa L\delta_0  \sigma^2}{T} \right)^{\nicefrac{1}{2}} +  \left( \frac{3L\delta_0 \sigma^2}{(n-f)T} \right)^{\nicefrac{1}{2}}  + 18\kappa G^2+ \frac{\Gamma_0 }{\gamma T} . 
		\end{split}
	\end{equation*}
\end{proof}

\bibliography{robustDL}

\begin{thebibliography}{10}
\providecommand{\url}[1]{#1}
\csname url@samestyle\endcsname
\providecommand{\newblock}{\relax}
\providecommand{\bibinfo}[2]{#2}
\providecommand{\BIBentrySTDinterwordspacing}{\spaceskip=0pt\relax}
\providecommand{\BIBentryALTinterwordstretchfactor}{4}
\providecommand{\BIBentryALTinterwordspacing}{\spaceskip=\fontdimen2\font plus
\BIBentryALTinterwordstretchfactor\fontdimen3\font minus
  \fontdimen4\font\relax}
\providecommand{\BIBforeignlanguage}[2]{{%
\expandafter\ifx\csname l@#1\endcsname\relax
\typeout{** WARNING: IEEEtran.bst: No hyphenation pattern has been}%
\typeout{** loaded for the language `#1'. Using the pattern for}%
\typeout{** the default language instead.}%
\else
\language=\csname l@#1\endcsname
\fi
#2}}
\providecommand{\BIBdecl}{\relax}
\BIBdecl

\bibitem{mcmahan2017communication}
B.~McMahan, E.~Moore, D.~Ramage, S.~Hampson, and B.~A. y~Arcas,
  ``Communication-efficient learning of deep networks from decentralized
  data,'' in \emph{Proc. Int. Conf. Artif. Intell. Statist}.\hskip 1em plus
  0.5em minus 0.4em\relax PMLR, 2017, pp. 1273--1282.

\bibitem{kairouz2021advances}
P.~Kairouz, H.~B. McMahan, B.~Avent, A.~Bellet, M.~Bennis, A.~N. Bhagoji,
  K.~Bonawitz, Z.~Charles, G.~Cormode, R.~Cummings \emph{et~al.}, ``Advances
  and open problems in federated learning,'' \emph{Found. Trends Mach. Learn.},
  vol.~14, no. 1--2, pp. 1--210, 2021.

\bibitem{liu2024survey}
C.~Liu, N.~Bastianello, W.~Huo, Y.~Shi, and K.~H. Johansson, ``A survey on
  secure decentralized optimization and learning,'' \emph{arXiv preprint
  arXiv:2408.08628}, 2024.

\bibitem{NEURIPS2019_ec1c5914}
G.~Baruch, M.~Baruch, and Y.~Goldberg, ``A little is enough: Circumventing
  defenses for distributed learning,'' in \emph{Proc. Adv. Neural Inf. Process.
  Syst. (NeurIPS'19)}, vol.~32, 2019.

\bibitem{so2020byzantine}
J.~So, B.~G{\"u}ler, and A.~S. Avestimehr, ``Byzantine-resilient secure
  federated learning,'' \emph{IEEE J. Sel. Areas Commun.}, vol.~39, no.~7, pp.
  2168--2181, 2020.

\bibitem{KHJ22}
S.~P. Karimireddy, L.~He, and M.~Jaggi, ``Byzantine-robust learning on
  heterogeneous datasets via bucketing,'' in \emph{Proc. Int. Conf. Learn.
  Represent. (ICLR'22)}, 2022.

\bibitem{yang2024byzantine}
C.~Yang and J.~Ghaderi, ``Byzantine-robust decentralized learning via
  remove-then-clip aggregation,'' in \emph{Proc. AAAI Conf. Artif. Intell.},
  vol.~38, no.~19, 2024, pp. 21\,735--21\,743.

\bibitem{seide20141}
F.~Seide, H.~Fu, J.~Droppo, G.~Li, and D.~Yu, ``1-bit stochastic gradient
  descent and its application to data-parallel distributed training of speech
  {DNNs}.'' in \emph{Interspeech}, vol. 2014.\hskip 1em plus 0.5em minus
  0.4em\relax Singapore, 2014, pp. 1058--1062.

\bibitem{alistarh2017qsgd}
D.~Alistarh, D.~Grubic, J.~Li, R.~Tomioka, and M.~Vojnovic, ``Qsgd:
  Communication-efficient sgd via gradient quantization and encoding,''
  \emph{Proc. Adv. Neural Inf. Process. Syst.}, vol.~30, 2017.

\bibitem{yin2018byzantine}
D.~Yin, Y.~Chen, R.~Kannan, and P.~Bartlett, ``Byzantine-robust distributed
  learning: Towards optimal statistical rates,'' in \emph{Proc. 35th Int. Conf.
  Mach. Learn.}\hskip 1em plus 0.5em minus 0.4em\relax PMLR, 2018, pp.
  5650--5659.

\bibitem{KHJ21}
S.~P. Karimireddy, L.~He, and M.~Jaggi, ``Learning from history for byzantine
  robust optimization,'' in \emph{Proc. 38th Int. Conf. Mach. Learn.}, ser.
  Proceedings of Machine Learning Research, vol. 139.\hskip 1em plus 0.5em
  minus 0.4em\relax PMLR, 18--24 Jul 2021, pp. 5311--5319.

\bibitem{el2021distributed}
E.~M. El~Mhamdi, R.~Guerraoui, and S.~L.~A. Rouault, ``Distributed momentum for
  byzantine-resilient stochastic gradient descent,'' in \emph{Proc. 9th Int.
  Conf. Learn. Represent. (ICLR)}, 2021.

\bibitem{FGGPS22}
S.~Farhadkhani, R.~Guerraoui, N.~Gupta, R.~Pinot, and J.~Stephan, ``{B}yzantine
  machine learning made easy by resilient averaging of momentums,'' in
  \emph{Proc. 39th Int. Conf. Mach. Learn.}, ser. Proceedings of Machine
  Learning Research, vol. 162.\hskip 1em plus 0.5em minus 0.4em\relax PMLR,
  17--23 Jul 2022, pp. 6246--6283.

\bibitem{alistarh2018convergence}
D.~Alistarh, T.~Hoefler, M.~Johansson, N.~Konstantinov, S.~Khirirat, and
  C.~Renggli, ``The convergence of sparsified gradient methods,'' \emph{Proc.
  Adv. Neural Inf. Process. Syst.}, vol.~31, 2018.

\bibitem{stich2018sparsified}
S.~U. Stich, J.-B. Cordonnier, and M.~Jaggi, ``Sparsified sgd with memory,''
  \emph{Proc. Adv. Neural Inf. Process. Syst.}, vol.~31, 2018.

\bibitem{bernstein2018signsgd}
J.~Bernstein, Y.-X. Wang, K.~Azizzadenesheli, and A.~Anandkumar, ``signsgd:
  Compressed optimisation for non-convex problems,'' in \emph{Proc. Int. Conf.
  Mach. Learn.}\hskip 1em plus 0.5em minus 0.4em\relax PMLR, 2018, pp.
  560--569.

\bibitem{ghosh2021communication}
A.~Ghosh, R.~K. Maity, S.~Kadhe, A.~Mazumdar, and K.~Ramchandran,
  ``Communication-efficient and byzantine-robust distributed learning with
  error feedback,'' \emph{IEEE J. Sel. Areas Commun.}, vol.~2, no.~3, pp.
  942--953, 2021.

\bibitem{acharya2022robust}
A.~Acharya, A.~Hashemi, P.~Jain, S.~Sanghavi, I.~S. Dhillon, and U.~Topcu,
  ``Robust training in high dimensions via block coordinate geometric median
  descent,'' in \emph{Proc. Artif. Intell. Statist.}\hskip 1em plus 0.5em minus
  0.4em\relax PMLR, 2022, pp. 11\,145--11\,168.

\bibitem{gorbunov2022variance}
E.~Gorbunov, S.~Horv{\'a}th, P.~Richt{\'a}rik, and G.~Gidel, ``Variance
  reduction is an antidote to byzantines: Better rates, weaker assumptions and
  communication compression as a cherry on the top,'' in \emph{Proc. Int. Conf.
  Learn. Represent. (ICLR'22)}, 2022.

\bibitem{zhu2023byzantine_TSP}
H.~Zhu and Q.~Ling, ``Byzantine-robust distributed learning with compression,''
  \emph{IEEE Trans. Signal Inf. Process. Netw.}, vol.~9, pp. 280--294, 2023.

\bibitem{rammal2023communication}
A.~Rammal, K.~Gruntkowska, N.~Fedin, E.~Gorbunov, and P.~Richt{\'a}rik,
  ``Communication compression for byzantine robust learning: New efficient
  algorithms and improved rates,'' in \emph{Proc. Artif. Intell.
  Statist.}\hskip 1em plus 0.5em minus 0.4em\relax PMLR, 2024, pp. 1207--1215.

\bibitem{guerraoui2024byzantine}
R.~Guerraoui, N.~Gupta, and R.~Pinot, ``Byzantine machine learning: A primer,''
  \emph{ACM Comput. Surv.}, vol.~56, no.~7, pp. 1--39, 2024.

\bibitem{AFGGPS23}
Y.~Allouah, S.~Farhadkhani, R.~Guerraoui, N.~Gupta, R.~Pinot, and J.~Stephan,
  ``Fixing by mixing: A recipe for optimal byzantine ml under heterogeneity,''
  in \emph{Proc. Artif. Intell. Statist. (AISTATS'23)}, ser. Proceedings of
  Machine Learning Research, vol. 206.\hskip 1em plus 0.5em minus 0.4em\relax
  PMLR, 25--27 Apr 2023, pp. 1232--1300.

\bibitem{yi2024near}
Y.~Yi, R.~You, H.~Liu, C.~Liu, Y.~Wang, and J.~Lv, ``Near-optimal resilient
  aggregation rules for distributed learning using 1-center and 1-mean
  clustering with outliers,'' in \emph{Proc. AAAI Conf. Artif. Intell.},
  vol.~38, no.~15, 2024, pp. 16\,469--16\,477.

\bibitem{hong2022hierarchical}
S.~Hong, H.~Yang, and J.~Lee, ``Hierarchical group testing for byzantine attack
  identification in distributed matrix multiplication,'' \emph{IEEE J. Sel.
  Areas Commun.}, vol.~40, no.~3, pp. 1013--1029, 2022.

\bibitem{yang2021basgd}
Y.-R. Yang and W.-J. Li, ``Basgd: Buffered asynchronous sgd for byzantine
  learning,'' in \emph{Proc. 38th Int. Conf. Mach. Learn.}\hskip 1em plus 0.5em
  minus 0.4em\relax PMLR, 2021, pp. 11\,751--11\,761.

\bibitem{wu2020federated}
Z.~Wu, Q.~Ling, T.~Chen, and G.~B. Giannakis, ``Federated variance-reduced
  stochastic gradient descent with robustness to byzantine attacks,''
  \emph{IEEE Trans. Signal Process.}, vol.~68, pp. 4583--4596, 2020.

\bibitem{LXCGL19}
L.~Li, W.~Xu, T.~Chen, G.~B. Giannakis, and Q.~Ling, ``{RSA}: Byzantine-robust
  stochastic aggregation methods for distributed learning from heterogeneous
  datasets,'' \emph{Proc. AAAI Conf. Artif. Intell. (AAAI'19)}, vol.~33, no.~1,
  pp. 1544--1551, Jul. 2019.

\bibitem{DD21}
D.~Data and S.~Diggavi, ``Byzantine-resilient high-dimensional sgd with local
  iterations on heterogeneous data,'' in \emph{Proc. 38th Int. Conf. Mach.
  Learn.}, ser. Proceedings of Machine Learning Research, vol. 139.\hskip 1em
  plus 0.5em minus 0.4em\relax PMLR, 18--24 Jul 2021, pp. 2478--2488.

\bibitem{karimireddy2019error}
S.~P. Karimireddy, Q.~Rebjock, S.~Stich, and M.~Jaggi, ``Error feedback fixes
  signsgd and other gradient compression schemes,'' in \emph{Proc. 36th Int.
  Conf. Mach. Learn.}\hskip 1em plus 0.5em minus 0.4em\relax PMLR, 2019, pp.
  3252--3261.

\bibitem{horvath2023stochastic}
S.~Horv{\'a}th, D.~Kovalev, K.~Mishchenko, P.~Richt{\'a}rik, and S.~Stich,
  ``Stochastic distributed learning with gradient quantization and
  double-variance reduction,'' \emph{Optim. Methods Softw.}, vol.~38, no.~1,
  pp. 91--106, 2023.

\bibitem{richtarik2021ef21}
P.~Richt{\'a}rik, I.~Sokolov, and I.~Fatkhullin, ``Ef21: A new, simpler,
  theoretically better, and practically faster error feedback,'' \emph{Proc.
  Adv. Neural Inf. Process. Syst.}, vol.~34, pp. 4384--4396, 2021.

\bibitem{lamport2019byzantine}
L.~Lamport, R.~Shostak, and M.~Pease, \emph{The Byzantine Generals
  Problem}.\hskip 1em plus 0.5em minus 0.4em\relax New York, NY, USA:
  Association for Computing Machinery, 2019, p. 203–226.

\bibitem{liu2021approximate}
S.~Liu, N.~Gupta, and N.~H. Vaidya, ``Approximate {B}yzantine fault-tolerance
  in distributed optimization,'' in \emph{Proc. ACM Symp. Princ. Distrib.
  Comput. (PODC'21)}.\hskip 1em plus 0.5em minus 0.4em\relax ACM, 2021, pp.
  379--389.

\bibitem{beznosikov2023biased}
A.~Beznosikov, S.~Horv{\'a}th, P.~Richt{\'a}rik, and M.~Safaryan, ``On biased
  compression for distributed learning,'' \emph{J. Mach. Learn. Res.}, vol.~24,
  no. 276, pp. 1--50, 2023.

\bibitem{fatkhullin2024momentum}
I.~Fatkhullin, A.~Tyurin, and P.~Richt{\'a}rik, ``Momentum provably improves
  error feedback!'' in \emph{Proc. Adv. Neural Inf. Process. Syst.}, 2023, pp.
  76\,444--76\,495.

\bibitem{li2021page}
Z.~Li, H.~Bao, X.~Zhang, and P.~Richt{\'a}rik, ``Page: A simple and optimal
  probabilistic gradient estimator for nonconvex optimization,'' in \emph{Proc.
  38th Int. Conf. Mach. Learn.}\hskip 1em plus 0.5em minus 0.4em\relax PMLR,
  2021, pp. 6286--6295.

\bibitem{huang2022lower}
X.~Huang, Y.~Chen, W.~Yin, and K.~Yuan, ``Lower bounds and nearly optimal
  algorithms in distributed learning with communication compression,''
  \emph{Proc. Adv. Neural Inf. Process. Syst.}, vol.~35, pp. 18\,955--18\,969,
  2022.

\bibitem{zhu2021broadcast}
H.~Zhu and Q.~Ling, ``Broadcast: Reducing both stochastic and compression noise
  to robustify communication-efficient federated learning,'' \emph{arXiv
  preprint arXiv:2104.06685}, 2021.

\bibitem{Chang2011libsvm}
C.-C. Chang and C.-J. Lin, ``{LIBSVM}: a library for support vector machines,''
  \emph{ACM Trans. Intell. Syst. Technol. (TIST)}, vol.~2, no.~3, pp. 1--27,
  2011.

\bibitem{CDWLJ+18}
S.~Caldas, S.~M.~K. Duddu, P.~Wu, T.~Li, J.~Kone{\v{c}}n{\`y}, H.~B. McMahan,
  V.~Smith, and A.~Talwalkar, ``Leaf: A benchmark for federated settings,''
  \emph{arXiv preprint arXiv:1812.01097}, 2018.

\bibitem{krizhevsky2012imagenet}
A.~Krizhevsky, I.~Sutskever, and G.~E. Hinton, ``Imagenet classification with
  deep convolutional neural networks,'' \emph{Proc. Adv. Neural Inf. Process.
  Syst.}, vol.~25, 2012.

\bibitem{PKH22}
K.~Pillutla, S.~M. Kakade, and Z.~Harchaoui, ``Robust aggregation for federated
  learning,'' \emph{IEEE Trans. Signal Process.}, vol.~70, pp. 1142--1154,
  2022.

\bibitem{tolpegin2020data}
V.~Tolpegin, S.~Truex, M.~E. Gursoy, and L.~Liu, ``Data poisoning attacks
  against federated learning systems,'' in \emph{Proc. 25th Eur. Symp. Res.
  Comput. Secur. (ESORICS 2020)}.\hskip 1em plus 0.5em minus 0.4em\relax
  Springer, 2020, pp. 480--501.

\bibitem{XKG20}
C.~Xie, O.~Koyejo, and I.~Gupta, ``Fall of empires: Breaking byzantine-tolerant
  sgd by inner product manipulation,'' in \emph{Proc. 35th Uncertainty Artif.
  Intell. Conf. (UAI’20)}, ser. Proceedings of Machine Learning Research,
  R.~P. Adams and V.~Gogate, Eds., vol. 115.\hskip 1em plus 0.5em minus
  0.4em\relax PMLR, 22--25 Jul 2020, pp. 261--270.

\end{thebibliography}
\bibliographystyle{IEEEtran}

\vfill

\end{document}